\documentclass[manuscript,screen,nonacm]{acmart}
\AtBeginDocument{%
  }

\setcopyright{acmlicensed}
\copyrightyear{2018}
\acmYear{2018}
\acmDOI{XXXXXXX.XXXXXXX}
\acmConference[Conference acronym 'XX]{Make sure to enter the correct
  conference title from your rights confirmation email}{June 03--05,
  2018}{Woodstock, NY}
\acmISBN{978-1-4503-XXXX-X/2018/06}

\usepackage{booktabs}
\usepackage{arydshln}
\usepackage{xcolor}
\usepackage{balance}
\usepackage{mathtools}
\usepackage{url}
\usepackage{amsfonts}
\usepackage{algorithmic}
\usepackage{textcomp}
\usepackage{algorithm}
\usepackage{multirow}
\usepackage{color}
\usepackage{subcaption}
\usepackage{enumitem}
\usepackage{graphicx}  
\usepackage{xfp} % Add this to your preamble
\usepackage{amsmath}

\newtheorem{lemma}{Lemma}

\theoremstyle{definition}

\theoremstyle{remark}

\usepackage{tikz}
\usepackage{pgfplots,pgfplotstable}
\usepackage[table]{xcolor}

\def\TRANSPA{100}

\definecolor{rank1}{HTML}{FED990}
\definecolor{rank2}{HTML}{D0D0D6}
\definecolor{rank3}{HTML}{E9B5A5}
\definecolor{rank4}{HTML}{FFFFFF}
\definecolor{rank5}{HTML}{FFFFFF}
\definecolor{rank6}{HTML}{FFFFFF}
\colorlet{rank1}{rank1!\TRANSPA}
\colorlet{rank2}{rank2!\TRANSPA}
\colorlet{rank3}{rank3!\TRANSPA}
\colorlet{rank4}{rank4!\TRANSPA}
\colorlet{rank5}{rank5!\TRANSPA}
\colorlet{rank6}{rank6!\TRANSPA}
\cellcolor{rank1}
\cellcolor{rank2}
\cellcolor{rank3}
\cellcolor{rank4}
\cellcolor{rank5}
\cellcolor{rank6}

\definecolor{rotate1}{HTML}{073B4C}
\definecolor{rotate2}{HTML}{009999}
\definecolor{rotate3}{HTML}{6062B8}
\definecolor{rotate4}{HTML}{0099FF}
\definecolor{rotate5}{HTML}{2149C9}
\definecolor{complex1}{HTML}{996633}
\definecolor{complex2}{HTML}{808000}
\definecolor{complex3}{HTML}{FF9900}
\definecolor{complex4}{HTML}{A50021}
\definecolor{complex5}{HTML}{CCCC00}

\usepackage{pdfrender}
\newcommand*{\boldcheckmark}{%
  \textpdfrender{
    TextRenderingMode=FillStroke,
    LineWidth=.7pt, % half of the line width is outside the normal glyph
  }{\checkmark}%
}

\pgfplotsset{compat=1.5.1}
\def\addlegendimage{\csname pgfplots@addlegendimage\endcsname}
\usetikzlibrary{patterns}
\usetikzlibrary{pgfplots.groupplots}
\usetikzlibrary{
  pgfplots.colorbrewer,
}
\pgfplotsset{
  cycle list/.define={my marks}{
    every mark/.append style={solid,fill=\pgfkeysvalueof{/pgfplots/mark list fill}},mark=*\\
    every mark/.append style={solid,fill=\pgfkeysvalueof{/pgfplots/mark list fill}},mark=square*\\
    every mark/.append style={solid,fill=\pgfkeysvalueof{/pgfplots/mark list fill}},mark=triangle*\\
    every mark/.append style={solid,fill=\pgfkeysvalueof{/pgfplots/mark list fill}},mark=diamond*\\
  },
}
\definecolor{aa}{rgb}{0.2,0.7,0.310}
\definecolor{cc}{rgb}{0.914,0.725,0.431}
\definecolor{bb}{rgb}{0.514,0.325,0.831}

\definecolor{xxx}{rgb}{0.6,0.80,0.98}
\definecolor{oxx}{rgb}{0.26,0.52,0.80}
\definecolor{oox}{rgb}{0.01,0.53,0.82}
\definecolor{ooo}{rgb}{0.00,0.23,0.6}

\definecolor{ooo_2}{HTML}{9b2948}
\definecolor{oox_2}{HTML}{ff7251}
\definecolor{oxx_2}{HTML}{ffca7b}
\definecolor{xxx_2}{HTML}{ffedbf}

\definecolor{comp}{HTML}{45C5F4}
\definecolor{la}{HTML}{92D050}
\definecolor{send}{HTML}{92D050}
\definecolor{wait}{HTML}{FFCD33}
\definecolor{recv}{HTML}{8D59B3}

\definecolor{s8}{HTML}{FF5349}
\definecolor{s16}{HTML}{FFCD33}
\definecolor{s32}{HTML}{92D050}
\definecolor{s64}{HTML}{45C5F4}
\definecolor{s128}{HTML}{8D59B3}

\newcommand{\method}{\textsf{PROBE}}
\newcommand{\E}{\mathcal{E}_i}
\newcommand{\codeurl}{\url{https://github.com/potato2734/probe-kgc-evaluation}}
\begin{document}

\title[Generalized Rank-based Evaluation for Knowledge Graph Completion: Perspectives, Framework, and Analyses]{Generalized Rank-based Evaluation for Knowledge Graph Completion: \\Perspectives, Framework, and Analyses}



% \affiliation{%

% \affiliation{%
%   \country{Iceland}

\author{Sooho Moon}
\orcid{0009-0001-1765-7688}
\affiliation{%
  \institution{Chung-Ang University}
  \city{Seoul}
  \country{South Korea}
}
\email{moonwalk725@cau.ac.kr}

\author{Jian Kang}
\orcid{0000-0003-3902-7131}
\affiliation{%
  \institution{Mohamed bin Zayed University of Artificial Intelligence}
  \city{Abu Dhabi}
  \country{UAE}
}
\email{jian.kang@mbzuai.ac.ae}

\author{Yunyong Ko}
\orcid{0000-0003-1283-4697}
\authornote{Corresponding author.}
\affiliation{%
  \institution{Chung-Ang University}
  \city{Seoul}
  \country{South Korea}
}
\email{yyko@cau.ac.kr}

% \affiliation{%

% \affiliation{%

% \affiliation{%

% \affiliation{%

\renewcommand{\shortauthors}{Moon et al.}


\definecolor{aa}{rgb}{0.2,0.7,0.310}
\definecolor{cc}{rgb}{0.914,0.725,0.431}
\definecolor{bb}{rgb}{0.514,0.325,0.831}

\definecolor{xxx}{rgb}{0.6,0.80,0.98}
\definecolor{oxx}{rgb}{0.26,0.52,0.80}
\definecolor{oox}{rgb}{0.01,0.53,0.82}
\definecolor{ooo}{rgb}{0.00,0.23,0.6}

\definecolor{ooo_2}{HTML}{9b2948}
\definecolor{oox_2}{HTML}{ff7251}
\definecolor{oxx_2}{HTML}{ffca7b}
\definecolor{xxx_2}{HTML}{ffedbf}

\definecolor{comp}{HTML}{45C5F4}
\definecolor{la}{HTML}{92D050}
\definecolor{send}{HTML}{92D050}
\definecolor{wait}{HTML}{FFCD33}
\definecolor{recv}{HTML}{8D59B3}

\definecolor{s8}{HTML}{FF5349}
\definecolor{s16}{HTML}{FFCD33}
\definecolor{s32}{HTML}{92D050}
\definecolor{s64}{HTML}{45C5F4}
\definecolor{s128}{HTML}{8D59B3}

\definecolor{rotate}{HTML}{073b4c}
\definecolor{complex}{HTML}{118ab2}
\definecolor{house}{HTML}{06d6a0}
\definecolor{tucker}{HTML}{FFB715}
\definecolor{plogicnet}{HTML}{F78C6B}
\definecolor{rnnlogic}{HTML}{ef476f}

\begin{abstract}
Knowledge graph completion (KGC) aims to predict missing facts from an observed knowledge graph (KG), playing a crucial role in a wide range of real-world applications such as drug discovery, recommender systems, and retrieval-augmented generation (RAG).
Although numerous KGC models have been proposed, the evaluation of KGC remains underexplored, despite its critical role in reliably assessing model performance and selecting appropriate models for real-world applications.
In this paper, we introduce two important perspectives for KGC evaluation that are overlooked by existing evaluation metrics,
\textbf{(P1)} \textit{predictive sharpness} — the degree of strictness in penalizing inaccurate predictions — and \textbf{(P2)} \textit{popularity-bias robustness} — the ability to evaluate predictions in a popularity-aware manner by accounting for both entity and relation popularity bias.
To address both perspectives, we propose a generalized evaluation framework, \textbf{{\method}}, 
which consists of a rank transformer (RT) that estimates the score of each prediction based on a desired level of predictive sharpness and a rank aggregator (RA) that determines the final evaluation score by aggregating all prediction scores according to a desired level of popularity-bias robustness.
We theoretically analyze {\method} by defining six key properties for reliable KGC evaluation and prove that {\method} satisfies all the properties, while existing metrics fail to satisfy some.
In particular, due to the open-world nature of KGs, an evaluation metric should preserve the relative performance of KGC models even when only incomplete facts are observed.
We show that {\method} better maintains such consistency, providing a more reliable estimate of intrinsic model performance than existing metrics.
Extensive experiments with six KGC models on six real-world KGs reveal that existing metrics may over- or under-estimate model performance depending on different evaluation perspectives, whereas {\method} enables a more comprehensive, flexible, and consistent evaluation of KGC models.
Our code and datasets are available at {\codeurl}.
\end{abstract}

\begin{CCSXML}
<ccs2012>
   <concept>
       <concept_id>10002951.10003227.10003351</concept_id>
       <concept_desc>Information systems~Data mining</concept_desc>
       <concept_significance>500</concept_significance>
       </concept>
   <concept>
       <concept_id>10010147.10010178.10010187</concept_id>
       <concept_desc>Computing methodologies~Knowledge representation and reasoning</concept_desc>
       <concept_significance>500</concept_significance>
       </concept>
 </ccs2012>
\end{CCSXML}

\ccsdesc[500]{Information systems~Data mining}
\ccsdesc[500]{Computing methodologies~Knowledge representation and reasoning}

\keywords{knowledge graphs, knowledge graph completion, rank-based evaluation, open-world assumption}

\maketitle

\section{Introduction}\label{sec:intro} % 2
A knowledge graph (KG) is a graph-structured representation of real-world knowledge, 
where an entity is represented as a node and a relation between two entities is represented as an edge in the form of a triple $(h,r,t)$. 
KGs~\cite{bollacker2008freebase,suchanek2007yago,miller1995wordnet,carlson2010toward,safavi2020codex} have been widely used in a wide range of applications such as question answering (QA)~\cite{hao2017end,yih2014semantic,huang2019kgqa-www,zhang2018kgqa-aaai,yasunaga2021kgqa-acl}, news classification~\cite{ko2023khan,feng2021nclass-arxiv,zhang2022nclass-naacl}, recommender systems~\cite{wang2018rec-www,wang2019rec-www,zhou2020rec-sigir}, drug discovery~\cite{zhang2021drug,zeng2022toward,bonner2022review,lin2020drug-ijcai}, and retrieval-augmented generation (RAG)~\cite{edge2024local,guo2020survey,pan2024unifying,chen2025pathrag,guo2024lightrag}. 
For example, in RAG, a KG can be used to better understand a user query and retrieve more query-relevant information, 
thereby enabling a large language model (LLM) to generate high-quality responses~\cite{luo2025hypergraphrag,matsumoto2024bioRAG,jiang2025medRAG}. 
In the pharmaceutical domain, biomedical KGs can be applied to drug discovery by providing useful information about complex relationships among diseases, genes, and compounds, which significantly reduces the cost and time required to develop new medications~\cite{zeng2022toward}.

Real-world KGs, however, are inherently incomplete~\cite{trouillon2016complex,sun2019rotate}, 
i.e., a number of facts are missing. 
This fundamental limitation can hinder the potential of KGs in real-world applications. 
To address this limitation, \textit{knowledge graph completion} (KGC) has been widely studied~\cite{bordes2013translating,yang2015embeddingentitiesrelationslearning,trouillon2016complex,sun2019rotate,qu2020rnnlogic,qu2019probabilistic,qi2024bi,liu2021indigo,zhang2022redgnn},
which aims to infer missing facts based on the observed KG structure.
Specifically, given a KG, it predicts the missing entity when either the head or the tail entity is unknown, i.e., in the form of $(h,r,?)$ or $(?,r,t)$. 
Despite recent breakthroughs in KGC models, little attention has been paid to the \textit{evaluation} of KGC models. 
Without appropriate evaluation, a suboptimal model may be selected, leading to low-quality knowledge and degraded performance in downstream tasks.

\begin{figure}[t]
    \centering
    \includegraphics[width=0.75\linewidth]{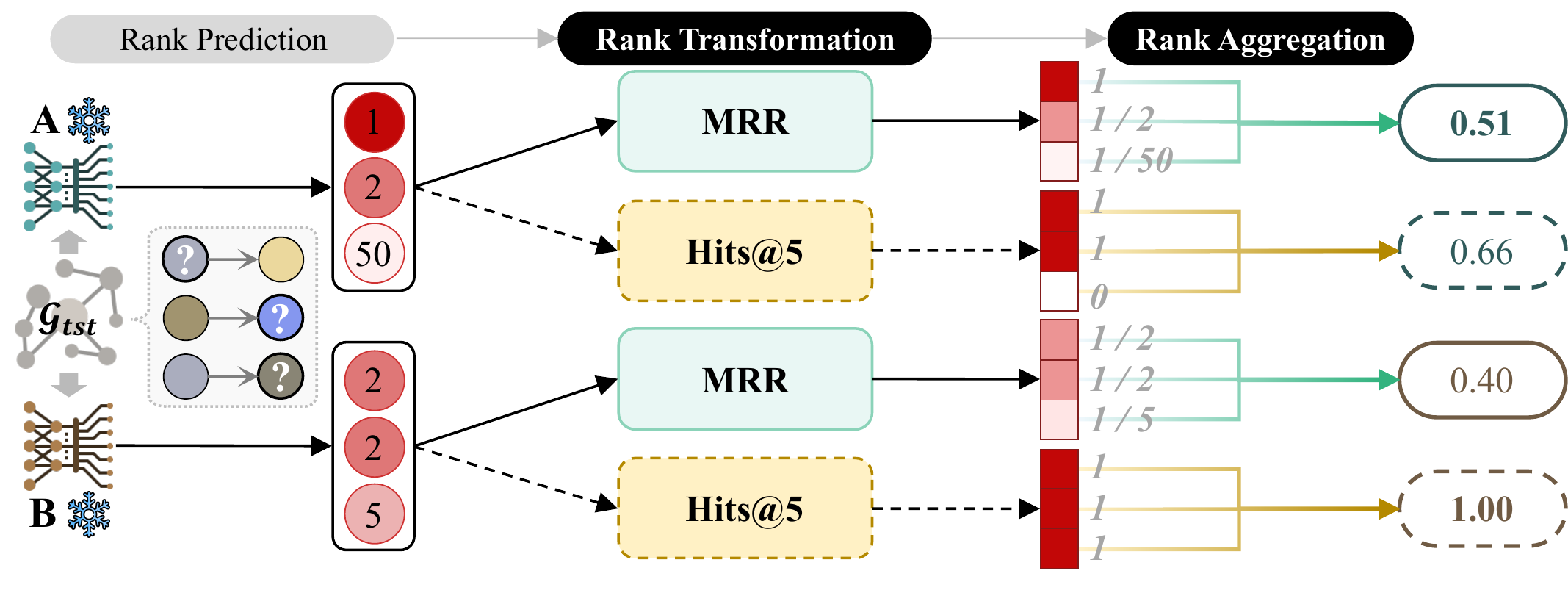}
    \vspace{-3mm}
    \caption{Rank-based evaluation protocol for knowledge graph completion: (1) prediction, (2) transformation, and (3) aggregation.}
    \vspace{-3mm}
    \label{fig:rank-based-evaluation}
\end{figure}

Motivated by this, we revisit the evaluation protocol for KGC models.
To evaluate KGC models, \textit{rank-based} evaluation metrics (e.g., mean rank (MR), mean reciprocal rank (MRR), and Hits@K) are commonly adopted~\cite{hoyt2022unified,yang2022rethinking,mohamed2020popularity,bordes2013translating,yang2015embeddingentitiesrelationslearning,moon2025sharp} due to the \textit{open-world assumption} (OWA)~\cite{yang2022rethinking,hoyt2022unified},
where the absence of a triple in a KG does not necessarily imply that it is false.
Specifically, the rank-based evaluation process is as follows:
Given a trained KGC model and a set of test triples, 
\textbf{(1)} (\textit{Prediction}) it computes the probability that each candidate entity is the missing entity for each triple – predicting the missing entity in the form of $(h,r,?)$ or $(?,r,t)$ – and ranks the candidates based on their probabilities; 
\textbf{(2)} (\textit{Transformation}) it transforms the rank of the missing entity into a prediction score; 
and \textbf{(3)} (\textit{Aggregation}) it aggregates the scores of all predictions to compute the final accuracy. 
Depending on the evaluation metric, the specific procedures for the (2) transformation and (3) aggregation steps can vary.

\vspace{1mm}
\noindent
\textbf{Examples.} 
As illustrated in Figure~\ref{fig:rank-based-evaluation}, 
consider two KGC models A and B evaluated on three test triples.
Suppose that model A ranks the target entities at $[1,2,50]$, whereas model B ranks them at $[2,2,5]$.
From the perspective of the evaluation procedure, 
their difference arises from how ranks are transformed into scores and aggregated.
For MRR, each rank is transformed using the reciprocal function $f(r)=\frac{1}{r}$.
Thus, the transformed scores are $\left[1, \frac{1}{2}, \frac{1}{50}\right]$ for model A and $\left[\frac{1}{2}, \frac{1}{2}, \frac{1}{5}\right]$ for model B.
In the aggregation step, the final score is computed as the average of the transformed scores:
$\mathrm{MRR}_A = \frac{1}{3} \left( 1 + \frac{1}{2} + \frac{1}{50} \right)=0.51,$ and $\mathrm{MRR}_B = \frac{1}{3} \left( \frac{1}{2} + \frac{1}{2} + \frac{1}{5} \right)=0.40$.
For Hits@5, the transformation step assigns a binary score depending on whether the rank of the target entity is within the top 5 predictions:
$f(r) = 1 \text{ if } r \le 5, \text{ otherwise } 0$.
Accordingly, the transformed scores of model A and model B are [1,1,0] and [1,1,1], respectively.
In the aggregation step, the final accuracy is computed as the average of the transformed scores:
$\mathrm{Hits@5}_A = \frac{1}{3} (1+1+0)=0.66$ and $\mathrm{Hits@5}_B = \frac{1}{3} (1+1+1)=1.00$.
This example shows that the preferred model depends on the evaluation metric:
MRR favors model A due to its stronger emphasis on \textit{top-ranked} predictions, 
whereas Hits@5 favors model B because it considers only whether the correct entity appears within the top-$K$ results.

\begin{figure}[t]
    \centering
    \includegraphics[width=0.98\linewidth]{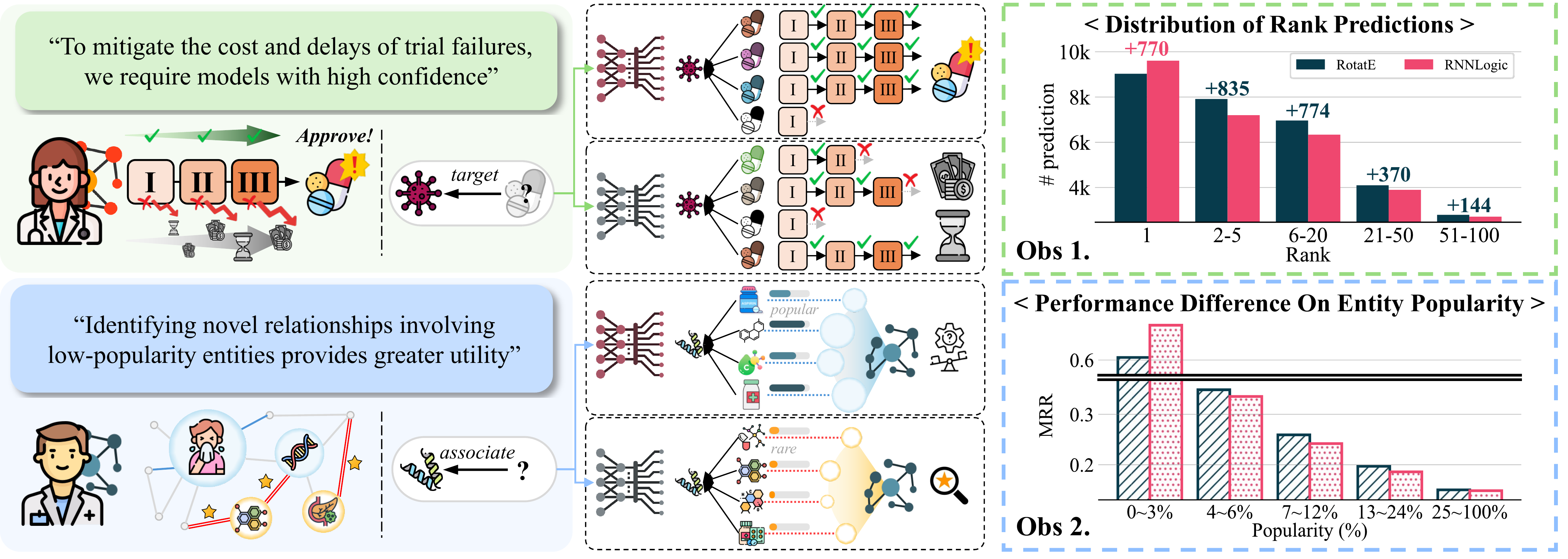}
    \vspace{-3mm}
    \caption{Motivating examples illustrating the importance of \textbf{(P1)} predictive sharpness and \textbf{(P2)} popularity-bias robustness in KGC evaluation. Depending on the desired evaluation perspective, different KGC models may be preferred.}
    \vspace{-3mm}
    \label{fig:motivation}
\end{figure}

\subsection{Motivation}
However, we observe that existing rank-based evaluation metrics overlook two subtle yet important perspectives for evaluating KGC models: (P1) \textit{predictive sharpness} and (P2) \textit{popularity-bias robustness}.
We present motivating examples and conduct preliminary experiments to justify the importance of these perspectives for KGC evaluation.

\vspace{1mm}
\noindent
\textbf{(P1)} \textbf{Predictive sharpness.} 
Naturally, we may answer the question “\textit{how strictly should we evaluate each individual prediction?}" differently depending on the application.
For example, when biomedical KGs are used for drug discovery, 
inaccurate (i.e., less faithful) predictions may lead to serious risks, 
such as harmful side effects, reduced effectiveness, and substantial clinical trial costs~\cite{zeng2022toward,vella2022medtrialcost}.
Thus, in this case, as predictions become less accurate, a larger penalty should be imposed on the KGC model.
In other words, a high level of \textit{predictive sharpness} is required in evaluating KGC models.
In contrast, when commonsense KGs are used in recommender systems, 
a relatively low level of predictive sharpness may be acceptable as long as the newly predicted facts are beneficial for better understanding users and items.

Existing metrics, however, do not consider predictive sharpness in KGC evaluation.
To verify our claim, 
we analyze the distribution of predictions from two state-of-the-art KGC models (RotatE~\cite{sun2019rotate} and RNNLogic~\cite{qu2020rnnlogic}).
As shown in Figure~\ref{fig:motivation}, RNNLogic produces more 1st-rank predictions, whereas RotatE performs better across ranks 2–100.
This implies that RNNLogic (resp. RotatE) is preferable when a high (resp. low) level of predictive sharpness is required.
However, existing metrics consistently assign higher scores to RNNLogic than RotatE.

% \begin{figure}[t]
% \centering
% \begin{tabular}{cc}
%     \includegraphics[width=0.465\linewidth]{figures/FB15k237_rank_hist_RotatE_RNNLogic.pdf} &
%     \includegraphics[width=0.465\linewidth]{figures/FB15k237_bar_RotatE_RNNLogic.pdf} \\
%     (a) Observation 1 & (b) Observation 2
% \end{tabular}
% \vspace{-3mm}
% \caption{Observations: Existing rank-based metrics have overlooked two subtle yet critical perspectives (P1) and (P2).}
% \vspace{-5mm}
% \label{fig:aspect1}
% \end{figure}

\vspace{1mm}
\noindent
\textbf{(P2)} \textbf{Popularity-bias robustness.} 
Real-world graphs generally follow a power-law degree distribution~\cite{wang2019tackling,xiong2018oneshot,ko2021mascot,jang2023sage}, meaning that most entities appear in only a few triples (low popularity), while a small number of entities appear in a large number of triples (high popularity). 
Thus, KGC models may exhibit \textit{popularity bias}, since high-popularity entities and relations are used much more frequently during training.
From an application perspective, the robustness of KGC models to popularity bias is often important in~\cite{bonner2022implications,ye2021medcoldstart}.
For example, in biomedical KGs for drug discovery, practitioners are typically more interested in discovering new facts associated with low-popularity entities and relations (e.g., rare diseases, less-investigated genes, or novel compounds), rather than repeatedly confirming well-known facts associated with high-popularity ones~\cite{bang2023biomedical,zhu2023mednewdisease,shomer2023toward}.
Thus, a high level of \textit{popularity-bias robustness} is required in KGC evaluation.

However, existing metrics fail to reflect this perspective and tend to overestimate KGC models that perform well on high-popularity triples.
To verify this, we measure the KGC accuracy (MRR) of RotatE and RNNLogic across different popularity levels.
As shown in Figure~\ref{fig:motivation}, 
both models achieve much higher accuracy on high-popularity triples, confirming strong popularity bias.
Interestingly, although RotatE outperforms RNNLogic across most popularity groups, RNNLogic achieves higher overall accuracy, indicating that existing metrics fail to properly reflect robustness to popularity bias in KGC evaluation.

\subsection{Our Work}
To reflect both perspectives \textbf{(P1)} and \textbf{(P2)}, we propose a generalized framework for KGC evaluation, 
named \textbf{\underline{P}}redictive sha\textbf{\underline{R}}pness and p\textbf{\underline{O}}pularity-\textbf{\underline{B}}ias robustness aware \textbf{\underline{E}}valuation (\textbf{{\method}}).
{\method} consists of a rank transformer (RT) and a rank aggregator (RA). 
For each test triple, the RT transforms the rank of the missing entity predicted by a KGC model into a score, based on the required level of predictive sharpness \textbf{(P1)}. 
Then, the RA assigns a weight to each test triple based on the desired level of popularity-bias robustness \textbf{(P2)}, 
and computes the final accuracy score by taking the weighted average of all transformed scores.
In particular, for \textbf{(P2)}, we observe that popularity bias varies across entities and relations.
This is because the popularity of a triple is jointly determined by entities and relations, i.e., even a high-popularity entity may rarely appear with a specific relation, and vice versa.
Based on this observation, {\method} separately measures \textbf{(P2)-(i)} entity popularity bias and \textbf{(P2)-(ii)} relation popularity bias for each test triple, and determines the final weight based on the two bias measures.
Note that existing metrics can be interpreted as special cases of {\method} under specific configurations of RT and RA.

In addition, we theoretically analyze that {\method} satisfies six key properties required for reliable KGC evaluation,
while existing metrics fail to satisfy some of them.
Importantly, under the open-world nature of KGs, accurately assessing the intrinsic performance of a KGC model is a critical requirement for evaluation metrics, i.e., a metric should preserve the relative performance of models even when only a partial set of facts is observed.
In other words, if two models have similar performance in the open world, their performance should remain consistent when evaluated in a closed-world setting where all facts are observable.
In this regard, we show that {\method} provides a reliable estimate of the intrinsic performance of KGC models in open-world settings.

Through comprehensive experiments with six KGC models on six real-world KGs,
we reveal the following findings.
\textbf{(1)} Existing rank-based metrics implicitly assume a high level of predictive sharpness, favoring KGC models that produce top-ranked predictions while underestimating those that produce imperfect yet high-quality predictions across broader rank ranges (e.g., 2nd--5th ranks).
\textbf{(2)} KGC models highly evaluated by existing metrics tend to exhibit strong popularity bias. 
Such models often perform poorly on low-popularity triples. 
These findings reveal inherent limitations of existing rank-based metrics in evaluating KGC models, as they overlook the crucial perspectives of predictive sharpness (P1) and popularity-bias robustness (P2).
In contrast, \textbf{(3)} {\method} enable more comprehensive evaluation of KGC models by flexibly reflecting different levels of predictive sharpness and popularity-bias robustness.
Finally, \textbf{(4)} {\method} consistently preserves the relative performance of KGC models across open-world and closed-world settings, demonstrating reliable evaluation under the open-world assumption.

\subsection{Contributions}
The main contributions of this work are summarized as follows.
\begin{itemize}
    \item \textbf{New Perspectives}: We introduce two key perspectives for KGC evaluation, \textbf{(P1)} predictive sharpness and \textbf{(P2)} popularity-bias robustness, which capture important aspects of KGC model performance overlooked by existing evaluation metrics.
    \item \textbf{Generalized Framework}: We propose a generalized evaluation framework, \textbf{{\method}}, that comprehensively evaluates KGC models by flexibly incorporating different levels of \textbf{(P1)} predictive sharpness and \textbf{(P2)} popularity-bias robustness.
    \item \textbf{Theoretical Analysis}: We provide a theoretical analysis showing that {\method} satisfies six key properties required for reliable KGC evaluation. We further analyze model evaluation under open-world and closed-world settings, demonstrating that {\method} preserves the relative performance of KGC models across the two settings.    
    \item \textbf{Comprehensive Experiments}: Through extensive experiments using six KGC models on six real-world KGs, we quantitatively and qualitatively demonstrate that {\method} provides more reliable and consistent evaluation results, particularly in preserving the relative performance of KGC models under the open-world assumption.
\end{itemize}

The paper is organized as follows. 
Section~\ref{sec:related} briefly reviews recent KGC models and existing evaluation metrics for KGC models. 
Section~\ref{sec:proposed} describes our evaluation framework, {\method}. 
Section~\ref{sec:analysis} theoretically analyzes key properties of {\method} and the ability to assess model performance under the open-world setting. 
Section~\ref{sec:eval} presents the experimental setup and results with in-depth analysis. 
Finally, Section~\ref{sec:con} concludes this paper.

% \begin{table}[t]
% \centering
% \caption{Notations and their descriptions}
% \setlength\tabcolsep{6pt}
% \begin{tabular}{cl}
% \toprule
%  \textbf{Notation} & \textbf{Description}\\
% \midrule

% $\mathcal{G}$ & a knowledge graph  \\
% $\mathcal{E}$ & a set of entities   \\
% $\mathcal{R}$ & a set of relations  \\
% $\mathcal{T}$ & a set of triples  \\

% \midrule
% % $\mathcal{D}$ & training, validation, and test sets  \\
% $\theta(\cdot)$ & a scoring function of a KGC model\\
% $f(\cdot)$ & a rank transformer function \\
% % $f^{*}(\cdot)$ & an affine transformed rank transformer function \\
% $agg(\cdot)$ & a rank aggregation function \\
% $w(\cdot)$ & an unit weight function \\
% % $N^{trn}_e$ & degree of entity e in the train graph \\
% % $N^{tst}_e$ & degree of entity e in the test graph \\
% \midrule
% $\mathbf{r}$ & a set of ranks\\
% $\mathbf{c}$ & a set of transformed scores\\
% $\mathbf{w}$ & a set of weights\\
% \midrule
% $\alpha$ & a factor of predictive sharpness \\
% $\beta$ & a factor of popularity-bias robustness\\
% % $W$ & weight partition function \\

% \bottomrule
% \end{tabular}
% \label{table:notations}
% \end{table}

% \begin{figure*}[t]
% \centering
% \includegraphics[width=1.0\textwidth]{figures/overview.pdf}
% \vspace{-2mm}
% \caption{Overview of {\method}, which consists of (1) prediction, (2) transformation, and (3) aggregation.}
% \vspace{-2mm}
% \label{fig:overview}
% \end{figure*}

\section{Related Work}\label{sec:related} % 2.5
In this section, we review common protocols and rank-based metrics for KGC evaluation, and introduce recent knowledge graph completion (KGC) models.

\subsection{Evaluation of Knowledge Graph Completion}
KGC models are commonly evaluated by rank-based metrics such as mean rank (MR), mean reciprocal rank (MRR), and Hits@K.
Despite their widespread adoption, these metrics have several limitations, 
including sensitivity to the open-world assumption, inability to account for popularity bias, and difficulty in comparing results across datasets.
To address these limitations, a handful of studies have revisited rank-based evaluation for KGC.
\citet{yang2022rethinking} showed that existing metrics yield inconsistent results under the open-world assumption and proposed alternative variants such as log-MRR and p-MRR.
However, these metrics do not address popularity bias.
\citet{mohamed2020popularity} introduced stratified metrics (e.g., strat-MRR and strat-Hits@K) that assign different weights to test triples based on entity and relation popularity.
While effective for mitigating popularity bias, they do not control the level of predictive sharpness.
\citet{berrendorf2020ambiguity} observed that rank-based scores are generally comparable only within the same dataset, since rank interpretation depends on the number of candidate entities.
They proposed the adjusted mean rank index (AMRI), a normalized version of MR that enables meaningful comparison across datasets.
\citet{hoyt2022unified} further extended AMRI to z-adjusted MRR (ZMRR), which normalizes scores using the mean and variance of the rank distribution.
Similarly, \citet{tiwari2021revisiting} proposed weighted geometric mean rank (WMR) that assigns higher importance to predictions made over larger candidate sets,
thereby reflecting the intuition that achieving the same rank in a more difficult setting should be more highly rewarded.

Those rank-based metrics are commonly computed under the filtered setting~\cite{bordes2013translating}, 
which removes candidate entities that correspond to other true triples when ranking the target entity. 
This prevents unfair penalization due to the incompleteness of real-world KGs, 
where multiple entities may form valid triples for the same query.
Another important protocol concerns tie-breaking~\cite{sun2020reevaluation}. 
Since multiple entities may receive identical scores from a KGC model, assigning the best possible rank to the target entity can lead to overly optimistic evaluation. 
To address this issue, ~\citet{sun2020reevaluation} proposed a tie-breaking protocol to assign the average rank within the tied group as the final rank, preventing expected overestimation of model performance.

Recent work~\cite{moon2025sharp} introduced two important perspectives for KGC evaluation: \textbf{(P1)} predictive sharpness and \textbf{(P2)} popularity-bias robustness.
The work proposed a new evaluation framework reflecting these two perspectives and analyzed several limitations of existing rank-based metrics.
However, its theoretical analysis was limited to only two properties, i.e., fixed optimum and fixed pessimum, without investigating the reliability of KGC evaluation under the open-world assumption (OWA).
In addition, popularity-bias robustness was modeled only using the frequency of the missing entity, limiting its ability to capture popularity bias induced by entity-relation interactions.
The experimental evaluation was also limited to four KGC models and two KGs.
In this paper, we substantially extend~\cite{moon2025sharp} in several directions.
First, we extend popularity-bias robustness by incorporating both entity popularity and entity-conditioned relation popularity.
Second, we provide more in-depth theoretical analysis on six key properties for reliable KGC evaluation as well as consistency under OWA.
Third, we conduct comprehensive experiments using six KGC models on six real-world KGs under diverse evaluation perspectives.

% To the best of our knowledge, 
% {\method} is the first generalized KGC evaluation framework to incorporate the two key perspectives, \textbf{(P1)} predictive sharpness and \textbf{(P2)} popularity-bias robustness.
% Existing metrics can be interpreted as special cases of {\method} under specific parameter settings corresponding to fixed levels of predictive sharpness and popularity-bias robustness.
% In addition, we provide theoretical analysis of the key properties of {\method} and compare them with those of existing metrics in Section~\ref{sec:analysis}.

\subsection{Knowledge Graph Completion Models}
Embedding-based approaches~\cite{bordes2013translating,sun2019rotate,li2022house,yang2015embeddingentitiesrelationslearning,trouillon2016complex,balazevic2019tucker} represent entities and relations as vectors in a latent space, preserving the semantic meaning of triples through algebraic operations.
TransE~\cite{bordes2013translating} models each relation as a translation vector such that the head entity translated by the relation is close to the tail entity.
RotatE~\cite{sun2019rotate} represents relations as rotations in complex space to capture diverse relation patterns such as symmetry and inversion.
HousE~\cite{li2022house} employs Householder transformations to model both rotation and projection for more expressive relations.
DistMult~\cite{yang2015embeddingentitiesrelationslearning} adopts a bi-linear scoring function with diagonal matrices to model symmetric relations.
ComplEx~\cite{trouillon2016complex} extends DistMult to complex-valued embeddings to capture asymmetric relations.
TuckEr~\cite{balazevic2019tucker} applies Tucker tensor decomposition to model rich interactions between entities and relations.

Rule-based approaches~\cite{yang2017differentiable,qu2019probabilistic,qu2020rnnlogic} learn logical patterns from relational paths (i.e., sequences of relations), enabling interpretable reasoning beyond observed triples.
NeuralLP~\cite{yang2017differentiable} adopts a recurrent architecture with attention and auxiliary memory to learn variable-length logical rules.
pLogicNet~\cite{qu2019probabilistic} employs a probabilistic logic neural network to learn logical rules while modeling their uncertainty.
RNNLogic~\cite{qu2020rnnlogic} treats logical rules as latent variables and jointly optimizes a rule generator and a reasoning predictor using the EM algorithm.
In addition, various deep learning-based approaches have also been studied~\cite{schlichtkrull2018modeling,vashishth2019composition,zhu2021neural,zhang2022redgnn,sun2024mlsaa}.
R-GCN~\cite{schlichtkrull2018modeling} extends graph convolutional networks (GCN) by applying relation-specific transformations.
CompGCN~\cite{vashishth2019composition} jointly learns entity and relation representations via compositional message passing.
NBFNet~\cite{zhu2021neural} applies a neural Bellman-Ford framework to encode path-based relational information.
RED-GNN~\cite{zhang2022redgnn} optimizes propagation in GNNs via dynamic programming for improved efficiency.
MLSAA~\cite{sun2024mlsaa} enhances inductive capability by integrating pretrained language models with graph neural networks through adaptive aggregation and multi-level sampling.

% In this paper, we comprehensively evaluate six state-of-the-art KGC models in six real-world KGs with different levels of predictive sharpness and popularity-bias robustness in Section~\ref{sec:eval}.

% Despite recent advances in KGC models, 
% they have focused primarily on learning representations from KGs for inferring missing facts, 
% while paying limited attention to \textit{how KGC models should be evaluated}. 
% Consequently, existing KGC approaches commonly rely on standard rank-based evaluation metrics such as mean rank (MR), mean reciprocal rank (MRR), and Hits@K without considering diverse evaluation perspectives.

\section{Proposed Framework: {\method}}\label{sec:proposed}
In this section, we present a generalized framework for KGC evaluation, 
named \textbf{{\method}} (\textbf{\underline{P}}redictive sha\textbf{\underline{R}}pness and p\textbf{\underline{O}}pularity-\textbf{\underline{B}}ias robustness aware \textbf{\underline{E}}valuation.
First, we introduce the notations used in this paper and formulate the problems that we consider. 
Then, we describe two key components of {\method}: 
a rank transformer (Section~\ref{sec:proposed-rt}) and a rank aggregator (Section~\ref{sec:proposed-ra}).
Finally, we present a geometric interpretation of {\method} (Section~\ref{sec:proposed-interpret}).

\begin{table}[t]
\caption{Notations and their descriptions}
\setlength\tabcolsep{7pt}% column space
\def\arraystretch{1.0} % row space
\centering
\begin{tabular}{cl}
\toprule
\textbf{Notation} & \textbf{Description} \\
\midrule
$\mathcal{G}=(\mathcal{E},\mathcal{R},\mathcal{T})$ & Knowledge graph consisting of entities, relations, and triples \\
$\mathcal{E}, \mathcal{R}, \mathcal{T}$ & Sets of entities, relations, and triples \\
$(h,r,t)$ & Triple with head entity $h$, relation $r$, and tail entity $t$ \\
$n$ & Number of test triples times two ($n=|\mathcal{T}_{test}|*2$) \\
$\theta(\cdot)$ & KGC model that assigns a score to a triple \\
$f(\cdot), agg(\cdot)$ & Rank transformation and aggregation functions \\
$r \in \mathbb{N}$ & Rank of the correct entity among candidate entities \\
$\mathbf{r} = [r_1,\dots,r_n]$, $\mathbf{c} = [c_1,\dots,c_n]$  & Rank vector and transformed score vector \\
$\alpha$, $\beta$ & Parameters controlling predictive sharpness and popularity-bias robustness \\ 
$\delta_e$, $\delta_{r|e}$ & Entity popularity and entity-conditioned relation popularity \\
${w}_e,{w}_{r|e}$ & Entity weight and entity-conditioned relation weight \\
$\textbf{w}=[w_{1},\dots,w_{n}]$ & Weight vector for test triples \\
% $\textbf{w}_e=[w_{e_1},\dots,w_{e_Q}],\textbf{w}_{r|e}=[w_{{r|e}_1},\dots,w_{{r|e}_Q}]$ & Sets of entity and reltioan weights \\
\bottomrule
\end{tabular}
\label{table:notations}
\end{table}

\subsection{Notations and Problem Formulation}
The notations used in this paper are described in Table~\ref{table:notations}.
This work focuses on \textit{the evaluation of knowledge graph completion (KGC)}.
Given $k$ candidate models, each model is first trained to solve the KGC task, and then evaluated to determine which model more accurately predicts missing facts.
To formalize this process, we consider the following two related problems: (1) knowledge graph completion and (2) rank-based KGC evaluation.

\vspace{1mm}
\noindent
\textbf{\textsc{Problem 1}} (\textsc{Knowledge Graph Completion}).
Given a knowledge graph (KG) $\mathcal{G}=(\mathcal{E},\mathcal{R},\mathcal{T})$,
where $\mathcal{E}$ is the set of entities, $\mathcal{R}$ is the set of relations, and $\mathcal{T}=\{((h,r,t)| h,t \in \mathcal{E}, r \in \mathcal{R}\}$ is the set of triples (i.e., facts),
the goal of knowledge graph completion (KGC) is to infer missing facts based on the observed KG $\mathcal{G}$.

\vspace{1mm}
Given a KGC model $\theta(h,r,t)$, 
its model parameters are typically trained based on the objective that encourage positive triples to obtain higher scores than negative triples generated through negative sampling. 
To achieve this objective, various loss functions have been used, including margin-based ranking loss~\cite{bordes2013translating}, logistic loss~\cite{trouillon2016complex,li2022house}, and cross-entropy loss~\cite{trouillon2016complex,balazevic2019tucker,qu2020rnnlogic}.

\vspace{1mm}
\noindent
\textbf{\textsc{Problem 2}} (\textsc{Rank-based KGC Evaluation}).
Given a trained KGC model $\theta(\cdot)$ and a set of test triples $\mathcal{T}_{test}$, 
the goal of rank-based KGC evaluation is to measure how accurately the model predicts missing entities for each triple (e.g., $(h,r,?)$ or $(?,r,t)$) under the open-world assumption.
Specifically, rank-based evaluation assesses model performance based on how highly the correct entity is ranked relative to other candidate entities across test triples, where candidates are ordered according to the scores produced by the model $\theta(\cdot)$.

Thus, an effective rank-based metric should appropriately reflect the relative ordering of predicted entities in the open-world setting and provide reliable assessment of model performance across different prediction scenarios.

\begin{figure}[t]
    \centering
    \includegraphics[width=1.0\linewidth]{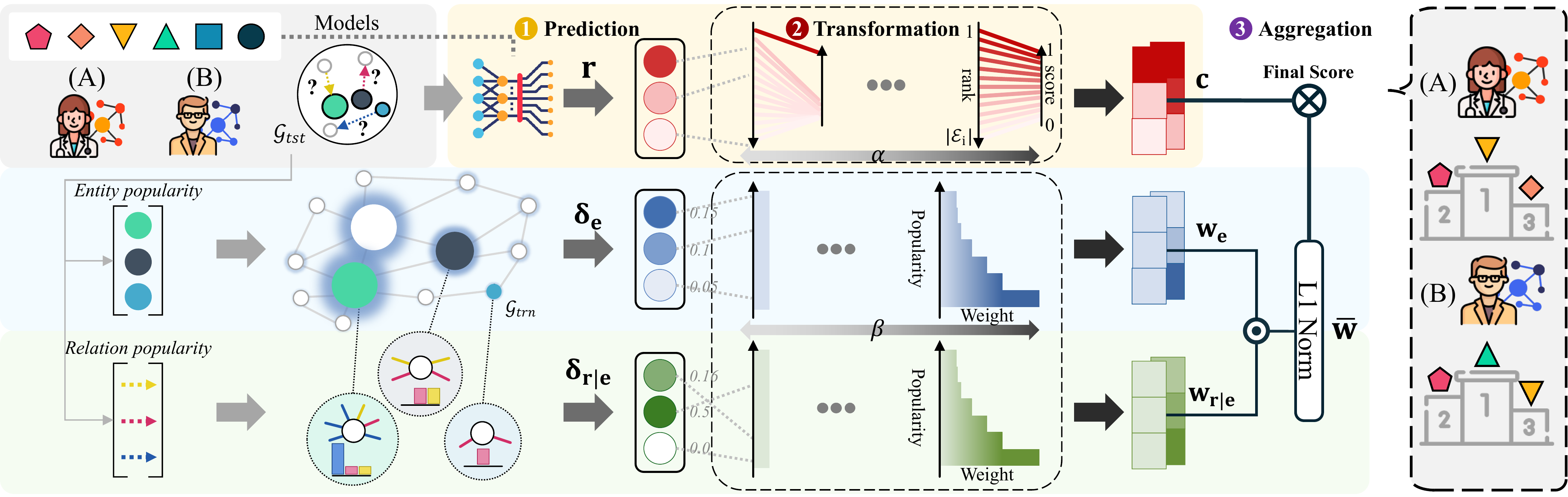}
    \vspace{-5mm}
    \caption{Overview of {\method}: (1) (\textit{prediction}) generating ranks for test triples using a KGC model; (2) (\textit{transformation}) converting ranks into scores via a predictive sharpness-aware function controlled by $\alpha$; and (3) (\textit{aggregation}) computing the final evaluation score using popularity-aware weights that capture triple-level popularity controlled by $\beta$.}
    \vspace{-2mm}
    \label{fig:overview}
\end{figure}

\vspace{1mm}
\noindent
\textbf{Overview of {\method}.}
Given a trained KGC model $\theta(\cdot)$ and a set of test triples $\mathcal{T}_{test}$,
{\method} evaluates the model $\theta$ based on the rank-based evaluation protocol~\cite{sun2020reevaluation,qu2020rnnlogic,li2022house,bordes2013translating}. 
As illustrated in Figure~\ref{fig:overview}, {\method} consists of the following three steps: (1) prediction, (2) transformation, and (3) aggregation.

\vspace{1mm}
\noindent
\textbf{(1) Prediction}:
For each test triple $(h,r,t)\in \mathcal{T}{test}$, two queries are generated by masking either the head or tail entity, i.e., $(h,r,?)$ and $(?,r,t)$.
A KGC model $\theta(\cdot)$ then estimates the likelihood of each candidate entity $e'\in\mathcal{E}$ being the missing entity.
Candidate entities are ranked according to their scores, producing the rank $1\le r\le|\mathcal{E}|$ of the correct entity.
For all queries, a rank vector $\mathbf{r}=[r_1,r_2,\dots,r_n]$ is produced, where $n=|\mathcal{T}_{{test}}|*2$ ($\because$ two queries for each triple).

\vspace{0.5mm}
\noindent
\textbf{(2) Transformation}:
Given $\mathbf{r}$, 
a rank transformation function $f(r,\alpha):\mathbb{N}\rightarrow\mathbb{R}$ converts each rank into a score based on the desired level of predictive sharpness $\alpha$.
The function $f(\cdot)$ is defined to be \textit{anti-monotone}, ensuring that lower numeric rank values yield higher scores.
This step produces a score vector $\mathbf{c}=[c_1,c_2,\dots,c_n]$.

\vspace{0.5mm}
\noindent
\textbf{(3) Aggregation}:
Given $\mathbf{c}$, a rank aggregation function
$agg(\mathbf{c},\beta):\mathbb{R}^n\rightarrow\mathbb{R}$
computes the final evaluation score as a weighted average of the scores in $\mathbf{c}$.
The weights are determined based on the popularity of the corresponding triples, reflecting the desired level of popularity-bias robustness controlled by $\beta$.

\subsection{Rank Transformation}\label{sec:proposed-rt}
The rank transformer (RT) of {\method} converts the rank of each prediction into a score, 
considering the required level of predictive sharpness. 
Formally, given a rank $r\in\mathbf{r}$ and a \textit{predictive sharpness control factor} $\alpha$, the score is computed as:
\begin{equation}
    f(r,\alpha)={1 \over r^{\alpha}} \hspace{1mm} (r \in\mathbf{r})\label{eq:transform}.
\end{equation}
When $\alpha<0$, the function becomes \textit{monotonically increasing}, assigning larger scores to worse predictions (larger ranks), 
leading undesirable evaluation behavior (e.g., sensitive to worse predictions)~\cite{hoyt2022unified,tiwari2021revisiting}.
Thus, we focus on $\alpha>0$, ensuring that lower ranks receive higher scores.
Figure~\ref{fig:RT}(a) illustrates the transformed scores for different values of $\alpha$.
Larger $\alpha$ imposes heavier penalties on large ranks, emphasizing top-ranked predictions.
Conversely, smaller $\alpha$ reduces the relative penalty on large ranks, allowing moderately ranked predictions to contribute more to the final score.

We note that this formulation can generalize various rank transformation rules used in existing metrics.
For example, when $\alpha=1$, the transformation reduces to the reciprocal rank $f(r,1)={1\over r}$ used in MRR.
As $\alpha=-1$, it degenerates to the identity function $f(r,-1)=r$, corresponding to directly using the rank values as in MR.
% Thus, existing metrics can be interpreted as special cases with a \textit{fixed} level of predictive sharpness.

\begin{figure}[t]
\centering
\begin{tabular}{cc}
    \includegraphics[width=0.26\linewidth]{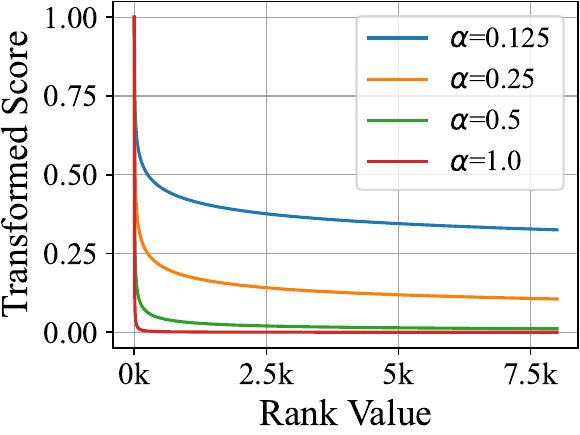} &
    \includegraphics[width=0.26\linewidth]{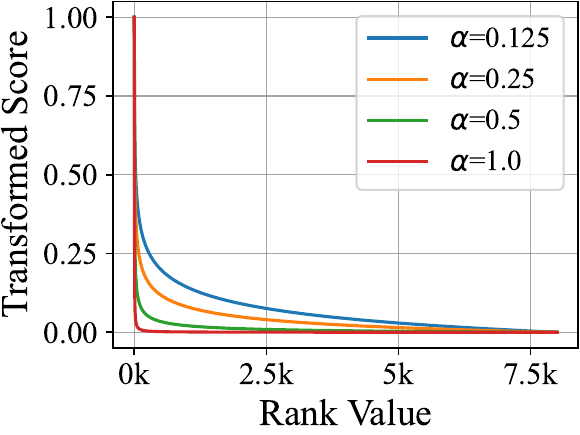} \\
    (a) Original RT  & (b) Affine RT \\
    \vspace{-7mm}
\end{tabular}
\caption{Rank transformers of {\method}: the affine RT rescales the range of scores to $[0,1]$.}
\vspace{-5mm}
\label{fig:RT}
\end{figure}

% \vspace{1mm}
\noindent
\textbf{\underline{Improving the distinguishability of RT}}.
As shown in Figure~\ref{fig:RT}(a), the lower bound of the RT increases as $\alpha$ approaches zero.
Consequently, the score range shrinks,
i.e., $|\E|^{-\alpha} \le f(r,\alpha) \le 1$, where $|\E|$ is the number of candidate entities for a test query \textit{i} in the filtered setting.
A narrower score range reduces the differences between scores corresponding to different ranks, 
limiting the ability of the metric to distinguish model performance.
This issue becomes more severe when $|\E|$ is small, as the pessimum (minimum) score increases further.
To address this limitation, we apply an \textit{affine} transformation to the RT to rescale the scores from $[|\E|^{-\alpha},1]$ to $[0,1]$ (see Figure~\ref{fig:RT}(b)), 
mapping the worst prediction to the pessimum score of $0$. % while preserving the relative ordering of ranks.
The rank transformer of {\method} is re-defined as:
\begin{equation}
    f^*(r_i,\alpha)=\frac{f(r_i,\alpha)-1}{1-(|\E|^{-\alpha})}+1\label{eq:transform-affine}.
\end{equation}
This formulation improves the distinguishability of {\method} by enlarging the effective score range while maintaining the consistency of rank order,
thereby enhancing its suitability as a rank-based evaluation metric for KGC models.

\subsection{Rank Aggregation}\label{sec:proposed-ra}
The rank aggregator (RA) of {\method} assigns a weight to each prediction based on the popularity of its corresponding triple.
A key question is \textit{``how to estimate the popularity of a target triple?"}
One possible way is to use the popularity of a target entity in a triple.
However, the popularity of a triple is often not fully captured by entity frequency alone.
% Even if an entity appears frequently across many triples, its association with a specific relation may still be sparse.
For example, \textit{Barack Obama} appears in a number of triples with relations such as \textit{bornIn}, \textit{presidentOf}, or \textit{educatedAt}, yet rarely with relations such as \textit{memberOfSportsTeam}.
Thus, the triple $(\text{Obama}, \textit{memberOfSportsTeam}, t)$ should be regarded as a low-popularity query despite the high popularity of the entity.
Similarly, relation frequency alone is insufficient to characterize triple popularity.
Some relations (e.g., \textit{locatedIn}, \textit{typeOf}, \textit{hasGender}) occur frequently overall, 
but may still be rare for specific entities.

To empirically verify our claim, we analyze the relationship between entity popularity and entity-conditioned relation popularity in real-world KGs.
Figure~\ref{fig:popularity} visualizes the entity and entity-conditioned relation popularity for triples sampled from three real-world KGs, where the bottom horizontal axis represents entity popularity $\delta(e)=d(e)/(|\mathcal{T}|*2)$, where $d(e)$ is the number of triples where entity $e$ appears and $|\mathcal{T}|$ is the total number of triples in a KG, 
and the top axis represents entity-conditioned relation popularity $\delta({r|e})=d(e\Rightarrow r)/d(e)$, where $d(e\Rightarrow r)$ is the number of triples where $e$ participate in the relation $r$.
Each triple is represented as a line connecting its entity popularity and entity-conditioned relation popularity values.
If entity popularity and relation popularity are strongly correlated, most lines would appear roughly parallel with few crossings.
However, we observe a large number of intersecting lines across all six datasets, indicating that entity popularity and relation popularity are not strongly correlated.

\begin{figure}[t]
\centering
\begin{tabular}{ccc}
    \includegraphics[width=0.3\linewidth]{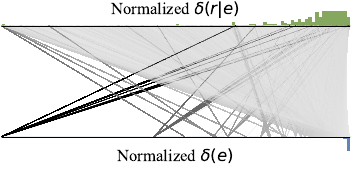} & \includegraphics[width=0.3\linewidth]{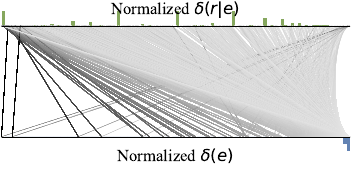} & \includegraphics[width=0.3\linewidth]{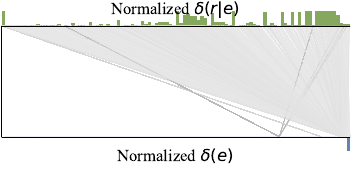} \\
    (a) FB15k237 & (b) WN18RR & (c) YAGO3-10 \\
    % \includegraphics[width=0.3\linewidth]{figures/family_e_r_corr.pdf} & \includegraphics[width=0.3\linewidth]{figures/umls_e_r_corr.pdf} & \includegraphics[width=0.3\linewidth]{figures/kinship_e_r_corr.pdf} \\
    % (d) Family & (e) UMLS & (f) Kinship \\
\end{tabular}
\vspace{-3mm}
\caption{Relationship between entity popularity $\delta(e)$ and entity-conditioned relation popularity $\delta(r|e)$ in three real-world KGs. The numerous intersecting lines indicate weak correlation between the two measures.}
\vspace{-5mm}
\label{fig:popularity}
\end{figure}

From this observation, we characterize popularity bias at the triple level by separately measuring (i) entity popularity and (ii) entity-conditioned relation popularity, and combining them to determine the final weight of a test query.
Formally, the weight of a test query $q=(h,r,?)$ (resp. $q=(?,r,t)$) is defined as:
\begin{align}
    w_t = w_e\cdot w_{r|e}, \quad \text{where} \quad w_e = {1 \over (\epsilon_e + \delta(e))^{\beta}}, \hspace{2mm} w_{r|e} = {1 \over (\epsilon_{r|e} + \delta(r|e))^{\beta}},\label{eq:weight}
\end{align}
where $e$ is the target entity,
$\delta(e)$ is the $e$'s popularity,
$\delta(r|e)$ is the popularity of the relation $r$ conditioned on $e$,
$\beta$ is a factor to control the level of popularity-bias robustness,
and $\epsilon$ is a small constant to prevent the division-by-zero problem.
Then, the weights for all test triples are normalized as follows: 
\begin{align}
\mathbf{\bar{w}} = { \mathbf{w} \over ||\mathbf{w}||_1 }, \quad \mathbf{w} = [w_1, w_2, \dots, w_n],
\end{align}
where $||\mathbf{w}_t||_1 = \sum_{i=1}^{n}w_i$ denotes the L1 norm.
This normalization ensures that {\method} computes a weighted average of the transformed scores, preventing the magnitude of the weights from affecting the final evaluation score.

Therefore, given a KGC model $\theta(\cdot)$ and a set $\mathcal{T}_{test}$ of test triples, {\method} computes the final evaluation score of the model by taking the weighted average of the transformed scores $\mathbf{c} \in \mathbb{R}^{n}$ produced by the RT, where each score is weighted by its corresponding popularity-aware weight $\mathbf{\bar{w}}\in \mathbb{R}^{n}$ determined by the RA:
\begin{align}
    {\method}(\theta, \mathcal{T}_{test})=agg(\mathbf{c},\mathbf{\bar{w}})={1 \over n}\sum^{n}_{i=1} w_i \cdot c_i.\label{eq:aggregation}
\end{align}

Figure~\ref{fig:RA}(a) shows the weight of an entity (resp. relation) according to their popularity across different levels of popularity-bias robustness $\beta$.
When $\beta>0$, RA assigns smaller weights to more popular entities (resp. relations) and relatively larger weights to less popular ones, mitigating the dominance of high-popularity triples in evaluation.
When $\beta=0$, all entities (resp. relations) receive the same weight of 1, reducing the evaluation to the simple average of transformed scores without considering popularity as in existing metrics such as MR, MRR, and Hits@K.
Figures~\ref{fig:RA}(b)-(d) show the final triple weights according to entity popularity $\delta(e)$ and entity-conditioned relation popularity $\delta(r|e)$.
Since $w_t$ is defined as the product $w_e \cdot w_{r|e}$, 
a triple is considered highly popular only when both popularity measures are high.
Thus, our weighting scheme captures popularity at the triple level.
% This formulation also generalizes weighting strategies used in existing metrics.
% For instance, when $\beta=0$, all predictions receive identical weights ($w_t=1$), corresponding to uniform averaging in MR and MRR.
% Therefore, existing metrics such as MR, MRR, and Strat-MRR can be interpreted as special cases with a \textit{fixed} level of popularity-bias robustness.

\begin{figure}[b]
\centering
\setlength\tabcolsep{0pt}
\begin{tabular}{cccccc}
    \includegraphics[width=0.22\linewidth]{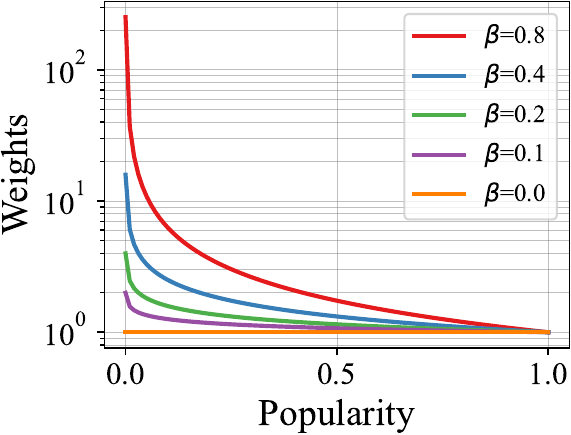} \qquad\qquad &
    \includegraphics[width=0.19\linewidth]{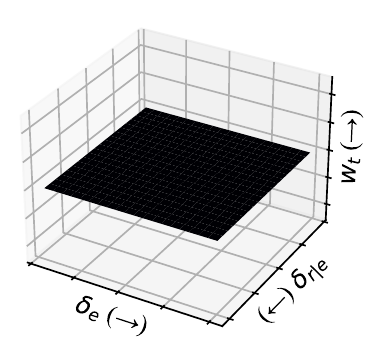} &
    \includegraphics[width=0.19\linewidth]{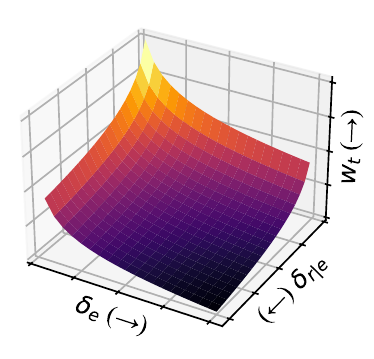} &
    \includegraphics[width=0.19\linewidth]{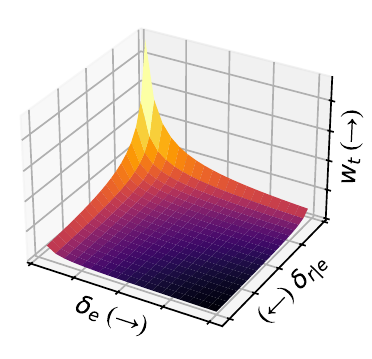} & 
    \includegraphics[width=0.03\linewidth]{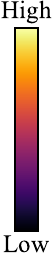}\\
    (a) Weight functions & (b) $\beta=0.0$ & (c) $\beta=0.1$ & (d) $\beta=0.8$ &  & \\
\end{tabular}
\vspace{-3mm}
\caption{(a) Weight functions with varying $\beta$ and (b) the final weights of triples depending on their entity and conditioned relation popularity with three different $\beta=0.0$, $\beta=0.1$, and $\beta=0.8$.}
\label{fig:RA}
\end{figure}

\subsection{Geometric Interpretation of {\method}}\label{sec:proposed-interpret}
In {\method}, the transformed scores and their corresponding popularity-aware weights can be viewed as vectors in an $n$-dimensional space, i.e., $\mathbf{c}=[c_1,c_2,\cdots,c_{n}]$ and $\mathbf{\bar{w}}=[\bar{w}_{1},\bar{w}_{2},\cdots,\bar{w}_{n}]$, where $n = 2*|\mathcal{T}_{test}|$ and each dimension corresponds to a test query.
Under this representation, the final evaluation score of a KGC model can be interpreted as the inner product between the two vectors:
\begin{align}
{\method}(\theta,\mathcal{T}_{test}) = agg(\mathbf{c},\mathbf{\bar{w}}) = \mathbf{c}\cdot \mathbf{\bar{w}}.
\label{eq:final-inner-product}
\end{align}

From this geometric viewpoint, 
the RT of {\method} maps the rank vector $\mathbf{r}$ into a transformed space whose geometry reflects the desired level of predictive sharpness.
Larger $\alpha$ amplifies score differences among top-ranked predictions, increasing the relative importance of correctly ranking entities near the best position (rank 1), 
resulting in sharper separation between top and lower ranks.
Conversely, smaller $\alpha$ compresses score differences across ranks, reducing the penalty for moderately ranked predictions.
Meanwhile, the weight vector $\mathbf{\bar{w}}$ defines an evaluation direction that reflects the desired level of popularity-bias robustness.
Each element $\bar{w}_i$ determines the contribution of the corresponding test query to the final score, indicating the relative importance of different regions in the evaluation space.
Larger $\beta$ tilts the evaluation direction toward low-popularity triples by assigning relatively larger weights to them.
Thus, $\alpha$ controls the geometry of the transformed score space, while $\beta$ controls the evaluation direction.

Under this interpretation, the final score measures the alignment between the score vector and the popularity-aware evaluation direction.
High evaluation performance can be achieved when model predictions align well with both evaluation perspectives: (P1) predictive sharpness and (P2) popularity-bias robustness.
Consequently, different KGC models may exhibit different degrees of alignment with this direction, leading to different evaluation outcomes.

\begin{figure}[t]
    \centering
    \includegraphics[width=1.0\linewidth]{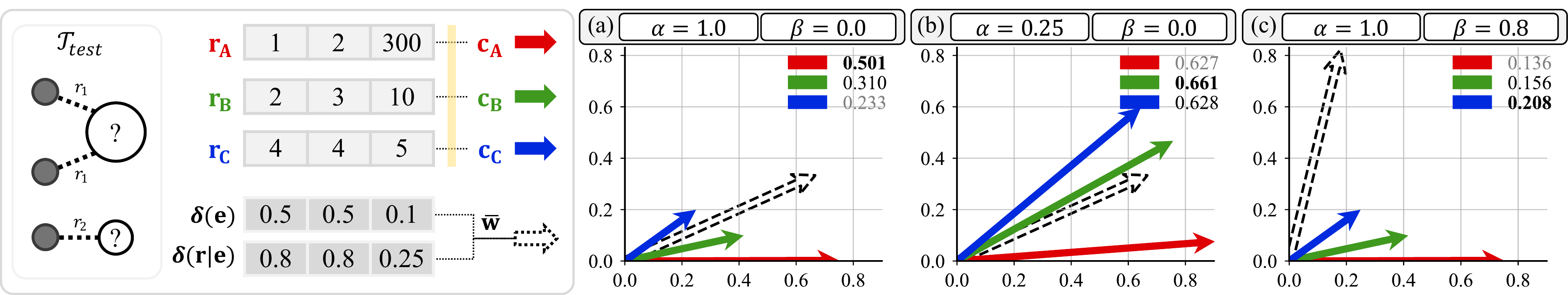}
    \caption{Geometric interpretation of {\method}: The parameter $\alpha$ controls the geometry of the transformed score space, while $\beta$ determines the evaluation direction through popularity-aware weighting. Therefore, the final evaluation score can be interpreted as the alignment between the score vector and the popularity-aware evaluation direction.}
    \label{fig:interpretability}
\end{figure}

\vspace{1mm}
\noindent
\textbf{Examples.} 
Consider three KGC models evaluated on three test triples.
They produce the rank vectors $\mathbf{r_A}=[1,2,300]$, $\mathbf{r_B}=[2,3,10]$, and $\mathbf{r_C}=[4,4,5]$.
Figure~\ref{fig:interpretability} shows that the final evaluation scores vary depending on the values of $\alpha$ and $\beta$, which can lead to different preferred models.
When $\alpha=1$ and $\beta=0$ (a), this setting imposes a strong penalty on incorrect predictions (i.e., high predictive sharpness) while not considering popularity bias. 
In this case, model A achieves the best performance.
When $\alpha=0.25$ and $\beta=0$ (b), the required level of predictive sharpness is reduced, allowing moderately ranked predictions to contribute positively to the overall evaluation score. 
As a result, a model that consistently predicts reasonable ranks, even if not top-ranked, may be preferred, and model B achieves the best performance.
When $\alpha=1$ and $\beta=0.8$ (c), the required level of popularity-bias robustness is increased, which assigns greater weights to the predictions on low-popularity triples.
In this case, a model that precisely predicts the ranks of low-popularity triples may be preferred, and model C achieves the best performance.

\section{Theoretical Analysis of {\method}}\label{sec:analysis}
In this section, we present six key properties for serving as a meaningful rank-based metric for KGC models and show that {\method} satisfies all the properties (Section~\ref{sec:analysis-property}). 
We further analyze how reliably our proposed metric can estimate the intrinsic predictive performance of KGC models under the open-world assumption (OWA) (Section~\ref{sec:analysis-openworld}).

\subsection{Key Properties of a Rank-based Metric}\label{sec:analysis-property}

\vspace{1mm}
\noindent
\textbf{\textsc{Property 1}} (\textsc{Fixed optimum}).
A rank-based metric satisfies the \textit{fixed optimum} property if its transformation function $f(\cdot)$ assigns a constant optimal score to the best possible prediction, i.e., when the target entity is ranked first by a KGC model: i.e., $f(r_i)=c_{opt}$, where $r_i=1$.

\vspace{0.5mm}
\textbf{Analysis under \textsc{Property 1}}.
{\method} satisfies this property since $f^*(1)=1=c_{opt}$ according to Eq.~(\ref{eq:transform-affine}).
MR, Hits@K, MRR, p-MRR, and Strat-MRR also satisfy this property, because each assigns its optimal value when $r_i=1$.
Specifically, MR achieves its minimum value, while Hits@K, MRR and the variants of MRR yield the maximum value of $1$.
Therefore, these metrics satisfy the fixed optimum property.

\vspace{1mm}
\noindent
\textbf{\textsc{Property 2}} (\textsc{Fixed pessimum}).
A rank-based metric satisfies the \textit{fixed pessimum} property if its transformation function $f(\cdot)$ assigns a constant pessimum score to the worst possible prediction, i.e., when the target entity is ranked last by a KGC model: i.e., $f(r_i)=c_{pes}$, where $r_i=|\E|$.

\vspace{0.5mm}
\textbf{Analysis under \textsc{Property 2}}.
{\method} satisfies this property since $f^*(|\E|)=0=c_{pes}$ according to Eq.~(\ref{eq:transform-affine}).
In contrast, MRR, p-MRR, and Strat-MRR do not guarantee a fixed pessimum score, because their transformation function $f(r_i)=1/r_i$ only approaches zero \textit{asymptotically} as $r_i \rightarrow \infty$.
Thus, these metrics do not provide a constant lower bound for the worst prediction.
In contrast, {\method} assigns a fixed pessimum value, improving interpretability and discrimination among poorly ranked predictions.

\vspace{1mm}
\noindent
\textbf{\textsc{Property 3}} (\textsc{Anti monotonicity}).
A rank-based metric is \textit{anti-monotone} if its transformation function $f(\cdot)$ assigns a higher score to a better rank, i.e., a lower rank value: i.e., $r_i > r_j \rightarrow f(r_i) < f(r_j)$.

\vspace{0.5mm}
\textbf{Analysis under \textsc{Property 3}}.
{\method} satisfies this property since the RT assigns higher scores to better ranks, i.e., for any $r_i > r_j$, $f(r_i,\alpha) < f(r_j,\alpha)$ (see Figure~\ref{fig:RT}).
MRR, p-MRR, and Strat-MRR also satisfy this property because their transformation function is monotonically decreasing with rank.
In contrast, MR does not satisfy this property since $f_{\mathrm{MR}}(r_i)=r_i$ assigns larger values to worse ranks.
Hits@K also violates this property because it assigns identical scores to different ranks within the same interval 
% (e.g., $f(r_i)=f(r_j)$ when both ranks are either $\le K$ or $>K$).
Although this property is not strictly required, it improves interpretability by ensuring that better predictions consistently receive higher scores within a bounded range (e.g., $[0,1]$).

\vspace{1mm}
\noindent
\textbf{\textsc{Property 4}} (\textsc{Candidate-size awareness}).
A rank-based metric is \textit{candidate-size aware} if its transformation function $f(\cdot)$ assigns different scores to the same rank depending on the size of the candidate entity set.
Formally, for two predictions with identical rank $r_i=r_j$ on two test triples with different candidate sizes $|\mathcal{E}_i|<|\mathcal{E}_j|$, it assigns different scores: i.e., $|\mathcal{E}_i| <|\mathcal{E}_j| \rightarrow f(r_i)<f(r_j)$.

\vspace{0.5mm}
\textbf{Analysis under \textsc{Property 4}}.
{\method} satisfies this property since the RT explicitly incorporates the candidate size $|\mathcal{E}_i|$ into the transformation function (Eq.~(\ref{eq:transform-affine})).
For the same rank $r$, predictions evaluated over larger candidate sets receive higher scores, reflecting the greater difficulty of achieving the same rank among more candidates.
Hits@K, MR, MRR, p-MRR, and Strat-MRR depend only on rank $r_i$ and ignore candidate size, assigning identical scores to predictions with the same rank.
In contrast, $\overline{WMR}$ and AMRI account for candidate size, enabling more meaningful comparison across predictions with different candidate spaces.
% This property captures the intuition that identical ranks should be evaluated differently depending on the size of the candidate set.

\begin{table*}[t]
\caption{Comparison of {\method} with existing metrics based on the six key properties.}
\vspace{-3mm}
\setlength\tabcolsep{5pt}
\begin{tabular}{lccccccc|c}
\toprule
Property  & MR & MRR & Hits@K & p-MRR & Strat-MRR & $\overline{WMR}$ & AMRI & \textbf{{\method}} \\

\midrule
\textsc{Fixed optimum} &  $\boldcheckmark$ & $\boldcheckmark$ & $\boldcheckmark$ & $\boldcheckmark$ & $\boldcheckmark$ & $\boldcheckmark$ & $\boldcheckmark$ & $\boldcheckmark$ \\

\textsc{Fixed pessimum} &  - & - & $\boldcheckmark$ & - & - & $\boldcheckmark$ & $\boldcheckmark$ & $\boldcheckmark$ \\

\textsc{Anti monotonicity} &  - & $\boldcheckmark$ & - & $\boldcheckmark$ & $\boldcheckmark$ & $\boldcheckmark$ & $\boldcheckmark$ & $\boldcheckmark$ \\

\textsc{Candidate-size awareness} & - & - & - & - & - & $\boldcheckmark$ & $\boldcheckmark$ & $\boldcheckmark$ \\

\textsc{Sharpness awareness} & - & - &$\boldcheckmark$& $\boldcheckmark$ & - & - & - & $\boldcheckmark$ \\

\textsc{Popularity awareness} &  - & - & - & - & - & - & - & $\boldcheckmark$ \\

\bottomrule
\end{tabular}
\vspace{-5mm}
\label{table:property}
\end{table*}

\vspace{1mm}
\noindent
\textbf{\textsc{Property 5}} (\textsc{Sharpness awareness}).
A rank-based metric is \textit{sharpness-aware} if its transformation function $f(\cdot)$ can adjust the sensitivity of evaluation to rank differences (i.e., varying levels of predictive sharpness).
Formally, for a fixed rank $r$ and two different sharpness levels $\alpha \neq \alpha'$, the metric assigns different scores: i.e., $f_{\alpha}(r) \neq f_{\alpha'}(r).$

\vspace{0.5mm}
\textbf{Analysis under \textsc{Property 5}}.
{\method} satisfies this property since the RT incorporates a controllable parameter $\alpha$ that adjusts the sensitivity of evaluation to rank differences.
Larger $\alpha$ imposes larger penalties on worse ranks, leading to sharper discrimination between high- and low-quality predictions.
Among existing metrics, p-MRR reflects predictive sharpness through its exponent parameter $p$.
However, p-MRR does not satisfy the fixed pessimum property (\textsc{Property 2}), as it cannot guarantee a constant lower bound for worst-ranked predictions.

% \vspace{0.5mm}
% \textbf{Analysis under \textsc{Property 5}}.
% {\method} satisfies this property as it incorporates an adjustable hyperparameter $\alpha$ to control how sharply incorrect predictions are penalized, i.e., larger $\alpha$ imposes steeper penalty for higher ranks (worse prediction), enabling sharper evaluation.
% p-MRR is the only existing metric that considers predictive sharpness.
% However, p-MRR does not satisfy the fixed pessimum property and cannot evaluate the popularity-bias robustness of a KGC model.
% Only three metrics can control the score mass using an additional parameter. 
% To increase predictive sharpness, Hits@K can increase K, p-MRR can decrease p.

\vspace{1mm}
\noindent
\textbf{\textsc{Property 6}} (\textsc{Popularity awareness}).
A rank-based metric is \textit{popularity-aware} if its aggregation function assigns different importance weights to prediction queries according to the popularity of their corresponding triples.
Formally, for a query triple $q \in \mathcal{T}_{test}$, the aggregation weight is determined by its popularity:
i.e., $w_q \propto \delta(q)$, where $\delta(q)$ denotes the popularity of query $q$.

\vspace{0.5mm}
\textbf{Analysis under \textsc{Property 6}}.
{\method} satisfies this property since its RA assigns weights based on the popularity of each query according to Eq.~(\ref{eq:weight}), assigning lower weights to high-popularity triples and higher weights to low-popularity triples.
This enables the evaluation metric to reflect the robustness of KGC models to popularity bias.
On the other hand, most existing metrics assign equal weights to all test triples and therefore ignore popularity differences across triples.
Although Strat-MRR introduces weighting across query groups, its weights depend on the grouping structure of the test set rather than the popularity structure of the KG, and thus cannot capture popularity-bias robustness.
In contrast, {\method} explicitly incorporates popularity information into the evaluation process, 
thereby enabling systematic assessment of model robustness to popularity bias.

As summarized in Table~\ref{table:property}, {\method} satisfies all six key properties, whereas existing metrics satisfy only a subset of them. 
This indicates that {\method} provides a more comprehensive and principled evaluation framework for KGC models, 
while remaining a generalized metric that can degenerate to existing metrics under specific parameter settings.

\subsection{Evaluation under the Open-World Assumption (OWA)}\label{sec:analysis-openworld}
Since real-world KGs are inherently incomplete~\cite{trouillon2016complex,sun2019rotate}, the absence of a triple does not necessarily imply that it is false.
Thus, an observed KG can be viewed as a \textit{partial observation} of an underlying complete KG containing all true facts, i.e., the closed-world assumption (CWA).
Therefore, an evaluation metric for KGC models should estimate the model’s intrinsic predictive performance of the complete KG, even when evaluated is performed on an incomplete KG, i.e., the open-world assumption (OWA).
In particular, if an incomplete KG is obtained by removing a subset of true facts from a complete KG, the intrinsic predictive performance of a model should remain unchanged.
Thus, the evaluation metric should provide consistent performance estimates across the two settings.
Formally, for a model $\theta$, a set of test triples $\mathcal{T}$, a desirable evaluation metric $M(\cdot)$ should satisfy:
% An evaluation metric should reflect the intrinsic predictive performance of a KGC model, independent of the world assumption adopted during evaluation.
% Ideally, the expected performance of a model should remain consistent whether the evaluation is conducted under the open-world assumption (OWA) or the closed-world assumption (CWA):
\begin{equation}
\mathbb{E}_{\mathcal{T} \sim \mathcal{D}_{OWA}}
\left[
M(\theta, \mathcal{T})
\right]
\approx
\mathbb{E}_{\mathcal{T} \sim \mathcal{D}_{CWA}}
\left[
M(\theta, \mathcal{T})
\right].
\end{equation}

However, accurate estimation of intrinsic predictive performance under the OWA remains challenging, 
because a number of true facts may be missing from the observed KG, and the extent of such missing facts is unknown.
While the filtered setting~\cite{bordes2013translating} improves evaluation reliability by removing known true entities during rank computation, it remains unclear how many additional true entities are still missing from the KG.
Consequently, the ranks computed by a KGC model may still underestimate the model’s performance since unobserved true entities may be ranked higher than the target entity but are incorrectly treated as negatives~\cite{lv2022crat1}.

Figure~\ref{fig:owa} illustrates this challenge with an example query (?, \textit{isLocatedIn}, Quebec), where the target answer is \textit{Rosemère}.
We train ComplEx~\cite{trouillon2016complex} on the YAGO3-10 dataset and make a prediction on the query.
\textbf{(1) \textit{Without} filtering}, the target entity is ranked 146th due to many higher-ranked true entities, resulting in a severely underestimated evaluation result.
\textbf{(2) \textit{With} filtering}, 137 known true entities, ranked higher than the target entity, are removed, thereby improving the rank to 9.
However, we observe that 7 of the remaining 8 higher-ranked entities (e.g., \textit{Yamachinche}, \textit{Maria-Chapdelanine}, and \textit{LaSalle}) are also true yet unknown.
After accounting for these unknown true entities (i.e., \textbf{(3) \textit{full} filtering}), 
the model ranks the target entity as 2nd among all true answers.
This example highlights that the intrinsic predictive performance of a KGC model should be evaluated based on the latent true rank (e.g., rank 2 in this example).

\begin{figure}[t]
    \centering
    \includegraphics[width=0.98\linewidth]{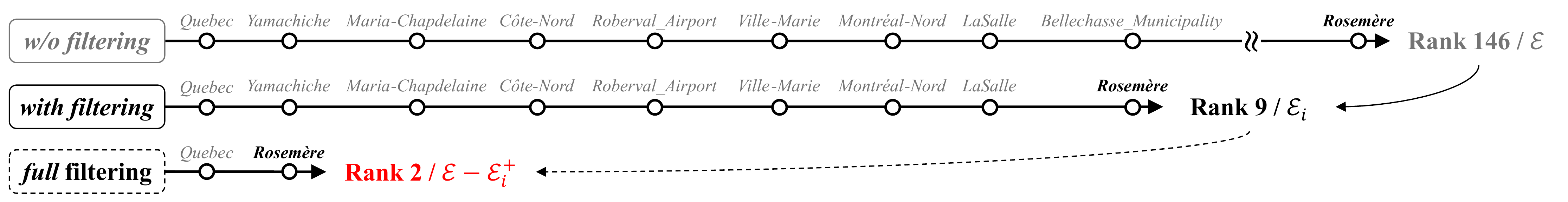}
    \vspace{-3mm}
    \caption{Example illustrating the challenge of evaluation under the OWA. 
    For the query (?, \textit{isLocatedIn}, Quebec), a model ranks the target entity \textit{Rosemère} as \textbf{(1)} 146th without filtering, \textbf{(2)} 9th after removing 137 known true entities with filtering, and \textbf{(3)} 2nd after additionally accounting for even unknown true entities (full filtering).}\label{fig:owa}
    \vspace{-5mm}
\end{figure}

Motivated by the aforementioned challenge of KGC evaluation under the OWA, 
\citet{yang2022rethinking} theoretically analyzed how reliably an evaluation metric reflects true model strength.
Specifically, they examined how the expected value of a metric changes with respect to the model’s true performance.
Formally, the derivative of the expected metric value with respect to true model strength $\ell$ is defined as:
\begin{equation}
\frac{\mathrm{d}\hat{\mathbb{E}}(M)}{\mathrm{d}\ell}
= \frac{1}{\ell \gamma (|\mathcal{E}_i^+|+1)}
\mathbb{E}_{r \sim \mathcal{B}(|\mathcal{E}_i^+|+1,\ell\gamma)}g(r),
\label{eq:cor4.1}
\end{equation}
where $M$ denotes the evaluation metric, $\ell$ represents the true model strength (i.e., model performance under CWA), $\gamma$ is the sparsity of the KG, $|\mathcal{E}_i^+|$ is the total number of true entities with respect to a given a test query \textit{i} in the complete KG, and $r$ is the predicted rank following a binomial distribution $\mathcal{B}(|\mathcal{E}_i^+|+1,\ell\gamma)$.
The function $g(r)=rf(r)$ depends on the rank transformation function $f(\cdot)$ defined by the metric.
Thus, Eq.~(\ref{eq:cor4.1}) provides a theoretical basis for analyzing how reliably an evaluation metric reflects intrinsic predictive performance under the OWA.

Now, we analyze how reliably {\method} reflects true model strength under the OWA.
Specifically, we analyze how the derivative of the expected {\method} score with respect to true model strength varies as a function of the sharpness control factor $\alpha$,
by examining the sign of the partial derivative of $\frac{\mathrm{d}\hat{\mathbb{E}}({\method})}{\mathrm{d}\ell}$ with respect to $\alpha$.
Since the denominator in Eq.~(\ref{eq:cor4.1}) and the binomial term $\mathcal{B}(|\mathcal{E}_i^+|+1,\ell\gamma)$ do not depend on $\alpha$,
the behavior of $\frac{\mathrm{d}\hat{\mathbb{E}}({\method})}{\mathrm{d}\ell}$ with respect to $\alpha$ is determined solely by $g(r)$.
Therefore, this analysis reduces to understanding the sign of $\frac{\partial}{\partial \alpha} g(r)$.

Then, we apply Eq.~(\ref{eq:cor4.1}) to the RT of {\method} and rewrite the derivative with respect to $\alpha$ as follows:
\begin{align}
\frac{\partial}{\partial\alpha}
\left(
\frac{\mathrm{d}\hat{\mathbb{E}}({\method})}{\mathrm{d}\ell}
\right)
&\propto \frac{\partial}{\partial\alpha} g(r) \notag
= \frac{\partial}{\partial\alpha} \big(r f_\alpha^*(r)\big) 
\propto
\frac{\partial}{\partial\alpha} f_\alpha^*(r) \quad (\because r \text{ is independent to } \alpha) \\
&= \frac{\partial}{\partial\alpha} (\frac{f(r,\alpha)-1}{1-(|\E|^{-\alpha})}+1) 
\quad (\because f_\alpha^*(r) =\frac{f(r,\alpha)-1}{1-(|\E|^{-\alpha})}+1) \\
&\propto (\frac{|\E|}{r})^\alpha ((r^\alpha-1)\ln |\E| + (1-|\E|^\alpha)\ln r).
\quad
\label{eq:probe-reduced}
\end{align}

Since $(\frac{|\E|}{r})^\alpha>0$,
the sign of
$\frac{\partial}{\partial\alpha}
\left(
\frac{\mathrm{d}\hat{\mathbb{E}}({\method})}{\mathrm{d}\ell}
\right)$
is determined by the sign of $F_\alpha(r)=(r^\alpha-1)\ln |\E| + (1-|\E|^\alpha)\ln r$.
By showing that $F_\alpha(r)$ does not change the sign, i.e., $\forall r$, $F_\alpha(r)\leq0$ or $F_\alpha(r)\geq0$,
we can ensure that the expected {\method} score changes \textit{monotonically} with respect to the true model strength $\ell$.
In other words, this indicates that \textit{\textbf{$\alpha$ directly controls how reliabily {\method} reflects intrinsic model performance under the OWA}}. 

From this observation, we present the following lemma:
\begin{lemma}
Let $\ell$ and $\alpha$ denote the true model strength under the CWA and the sharpness parameter, respectively.
Then, the derivative of the expected {\method} score with respect to $\ell$ is a monotonically decreasing function of $\alpha\in\mathbb{R}$ ($\alpha\neq0$):
\[
\frac{\partial}{\partial\alpha}
\left(
\frac{\mathrm{d}\hat{\mathbb{E}}({\method})}{\mathrm{d}\ell}
\right)
\le 0.
\]
\label{lem:probe}
\end{lemma}
\textsc{Proof for Lemma 1}.
This proof is equivalent to showing that $F_\alpha(r)\le 0$ for all $r<|\E|$.
The derivative of $F_\alpha(r)$ with respect to $\alpha$ is defined as:
\begin{align}
\frac{\partial}{\partial \alpha} F_\alpha(r) 
&= \frac{\partial}{\partial \alpha} ((r^\alpha-1)\ln |\E| + (1-|\E|^\alpha)\ln r) \notag\\
&= (r^\alpha-|\E|^\alpha)\ln r \ln |\E|. \label{eq:derivative-F}
\end{align}
We first identify the feasible domain of the rank $r$.
Let $\mathbf{m}\sim\mathcal{B}(|\mathcal{E}_i^+|,\gamma)$ denote the number of missing true entities among all the true entities $|\mathcal{E}_i^+|$.
Since the number of filtered true entities excluding the target entity is $|\mathcal{E}_i^+|-\mathbf{m}-1$,
and the number of remaining candidate entities is $|\E|=|\mathcal{E}|-(|\mathcal{E}_i^+|-\mathbf{m}-1)$.
Thus, we have $|\mathcal{E}|-|\mathcal{E}_i^+|+1\le |\E|$ ($\because \mathbf{m}\ge 0$).
Then, from the observation in~\cite{yang2022rethinking} that for any query \textit{i}, the total number of true entities is typically less than 10\% of the total candidate entities, i.e., $|\mathcal{E}_i^+|/|\mathcal{E}|<10\%$, 
we obtain the following inequality:
\begin{equation}
1\le r\le |\mathcal{E}_i^+|+1 < 9|\mathcal{E}_i^+|+1 < |\mathcal{E}|-|\mathcal{E}_i^+|+1 \le |\E|\le |\mathcal{E}| \\
.\label{eq:derived_ineq}
\end{equation}
Based on this inequality, by applying $r<|\E|$ and $1\leq |\E|$ to Eq.~\ref{eq:derivative-F}, we obtain the following sign of the derivative:
\begin{align}
\frac{\partial}{\partial \alpha} F_\alpha(r)
\begin{cases}
\le 0 & \text{if } \alpha>0 \\
=0 & \text{if } \alpha\to 0 \\
\ge 0 & \text{if } \alpha<0
\end{cases}.
\end{align}

This shows that $F_\alpha(r)$ remains non-positive over the domain of interest, i.e., $\forall r$, $F_\alpha(r)\le 0$.
Recall that since $\left(\frac{|\E|}{r}\right)^\alpha>0$, 
the sign of 
$\frac{\partial}{\partial\alpha}\left(\frac{\mathrm{d}\hat{\mathbb{E}}({\method})}{\mathrm{d}\ell}\right)$
is determined solely by $F_\alpha(r)$.
Therefore, we conclude that
\begin{align}
\frac{\partial}{\partial\alpha}
\left(
\frac{\mathrm{d}\hat{\mathbb{E}}({\method})}{\mathrm{d}\ell}
\right)
\le 0.
\end{align}
\qed

\section{Experimental Validation}\label{sec:eval} 
In this section, we aim to answer the following research questions: 
% \begin{itemize}[leftmargin=8pt]
\begin{itemize}
    \item \textbf{RQ1}. To what extent does (P1) the predictive sharpness affect the evaluation score of KGC models?
    \item \textbf{RQ2}. To what extent  does (P2) the popularity-bias robustness affect the evaluation score of KGC models?
    \item \textbf{RQ3}. Does {\method} provide a more comprehensive evaluation of KGC models, compared to existing metrics?
    \item \textbf{RQ4}. How reliable are the evaluation results produced by {\method} under the open-world assumption (OWA)?
\end{itemize}

\subsection{Experimental Setup}\label{sec:eval-setup}
\noindent
\textbf{KG datasets and KGC models.}
In our experiments, we use six real-world knowledge graphs (KGs), which were also used in~\cite{sun2019rotate,qu2020rnnlogic,trouillon2016complex,balazevic2019tucker,li2022house,dettmers2018convolutional,yang2015embeddingentitiesrelationslearning}:
FB15k237~\cite{toutanova2015observed}, WN18RR~\cite{miller1995wordnet},
YAGO3-10~\cite{mahdisoltani2013yago3}, Family~\cite{yang2017differentiable}, UMLS~\cite{nickel2011rescal}, and Kinship~\cite{nickel2011rescal}. 
Table~\ref{table:datasets} shows the data statistics.
We use six state-of-the-art KGC models:
4 embedding-based models (RotatE~\cite{sun2019rotate}, ComplEx~\cite{trouillon2016complex}, HousE~\cite{li2022house}, and TuckEr~\cite{balazevic2019tucker}) and 2 rule-based models (pLogicNet~\cite{qu2019probabilistic} and RNNLogic~\cite{qu2020rnnlogic}).
For all KGC models, we use the official source codes provided by the authors.
Also, we release the code and datasets used in this paper at {\codeurl}.

\begin{table}[h]
\centering
\caption{Statistics of six real-world KGs used in our experiments.}
\label{table:datasets}
\setlength\tabcolsep{10pt}
\def\arraystretch{1.1} % row space
\begin{tabular}{l|rrrrr}
\toprule

 & $|\mathcal{E}|$ & $|\mathcal{R}|$ & $|\mathcal{T}|$ & $\delta(q)_{avg}$ & $\delta(q)_{max}$ \\ 
 \midrule

 \textbf{FB15k237} & 14,541 & 237 & 272,115 & 37.4 & 7,614 \\
 \textbf{WN18RR} & 40,943 & 11 & 86,835 & 4.2 & 482 \\
  \textbf{YAGO3-10} & 123,182 & 37 & 1,079,040 & 17.5 & 61,044 \\
 \textbf{Family} & 3,007 & 12 & 23,483 & 15.6 & 94 \\
  \textbf{UMLS} & 135 & 46 & 1,959 & 29.0 & 140 \\
 \textbf{Kinship} & 104 & 25 & 3,206 & 61.7 & 80 \\
% Dataset & \textbf{FB15k237}~\cite{bollacker2008freebase} & \textbf{WN18RR}~\cite{miller1995wordnet} \\

% \midrule
% \textit{\# of entities}  & 14,541 & 40,943 2 \\
% \textit{\# of relations} & 237 & 11  \\
% \textit{\# of triples}   & 272,115 & 86,835  \\
% \textit{Avg. degree}   & 37.5 & 4.3  \\
% \textit{Max. degree}   & 7,614 & 482  \\

\bottomrule
\end{tabular}
\end{table}

% \vspace{1mm}
\noindent
\textbf{Evaluation protocol.}
We use the evaluation protocol exactly same as that used in~\cite{sun2020reevaluation,qu2020rnnlogic,zhang2022redgnn}. 
For each query, we remove its corresponding known true entities from candidate entities during ranking (i.e., the filtered setting)
and assign the average rank of the tied group as the final rank (i.e., the tie-breaking protocol) for fair evaluation.
% During the training, the accuracy of a KGC model is measured on the validation set at every epoch and the best-performing model is saved.
% Then, we evaluate the best-performing model by using {\method} with varying levels of predictive sharpness and popularity-bias robustness.
During training, a KGC model is optimized to minimize its specific loss function (e.g., margin-based ranking loss, logistic loss, or cross-entropy loss). 
At each epoch, the model is evaluated on the validation set using a target evaluation metric (e.g., MRR, p-MRR, or {\method}), and the best-performing model is saved. 
Then, the selected model is evaluated on the test set using the same evaluation metric.
We report the averaged performance over three runs with different random seeds.
% {\method} under varying levels of predictive sharpness and popularity-bias robustness.

\begin{table}[t]
\centering
\footnotesize
\caption{The accuracy of KGC Models evaluated by {\method} with varying levels of predictive sharpness $\alpha$. The \textcolor[HTML]{F8A602}{gold}, \textcolor[HTML]{A5A5B1}{silver}, and \textcolor[HTML]{D36A49}{bronze} indicate the best, the second, and the third best results, respectively.}
\vspace{-3mm}
% The top three performers in each category are highlighted with a color gradient.} 
\setlength{\tabcolsep}{3.0pt}
\begin{tabular}{l ccccc | ccccc | ccccc}
\toprule
\toprule
\multirow{3}{*}{Models} 
    & \multicolumn{5}{c}{FB15k-237} 
    & \multicolumn{5}{c}{WN18RR}
    & \multicolumn{5}{c}{YAGO3-10}\\
\cmidrule(lr){2-16} 
    & $\alpha$=1.0 & $\alpha$=0.5 & $\alpha$=0.0 & $\alpha$=-0.5 & $\alpha$=-1.0   
    & $\alpha$=1.0 & $\alpha$=0.5 & $\alpha$=0.0 & $\alpha$=-0.5 & $\alpha$=-1.0    
    & $\alpha$=1.0 & $\alpha$=0.5 & $\alpha$=0.0 & $\alpha$=-0.5 & $\alpha$=-1.0   \\
\midrule

RotatE & \cellcolor{rank5}{0.3140} & \cellcolor{rank4}{0.4241} & \cellcolor{rank4}{0.7124} & \cellcolor{rank4}{0.9380} & \cellcolor{rank4}{0.9844}
 & \cellcolor{rank4}{0.4641} & \cellcolor{rank3}{0.5213} & \cellcolor{rank3}{0.6952} & \cellcolor{rank3}{0.8549} & \cellcolor{rank3}{0.9017}
& \cellcolor{rank2}{0.4743} & \cellcolor{rank2}{0.5640} & \cellcolor{rank3}{0.7971} & \cellcolor{rank3}{0.9555} & \cellcolor{rank3}{0.9813}
 
        \\
ComplEx & \cellcolor{rank4}{0.3146} & \cellcolor{rank3}{0.4283} & \cellcolor{rank2}{0.7213} & \cellcolor{rank2}{0.9425} & \cellcolor{rank2}{0.9855}
 & \cellcolor{rank3}{0.4754} & \cellcolor{rank4}{0.5179} & \cellcolor{rank4}{0.6594} & \cellcolor{rank4}{0.8121} & \cellcolor{rank4}{0.8700}
 & \cellcolor{rank5}{0.4496} & \cellcolor{rank3}{0.5467} & \cellcolor{rank2}{0.7979} & \cellcolor{rank1}{0.9588} & \cellcolor{rank2}{0.9815}

        \\

HousE & \cellcolor{rank6}{0.3073} & \cellcolor{rank6}{0.4100} & \cellcolor{rank6}{0.6910} & \cellcolor{rank5}{0.9244} & \cellcolor{rank5}{0.9778}
 & \cellcolor{rank6}{0.4184} & \cellcolor{rank6}{0.4913} & \cellcolor{rank2}{0.6991} & \cellcolor{rank2}{0.8823} & \cellcolor{rank1}{0.9346}
 & \cellcolor{rank3}{0.4538} & \cellcolor{rank5}{0.5297} & \cellcolor{rank5}{0.7500} & \cellcolor{rank5}{0.9267} & \cellcolor{rank5}{0.9647}
 
        \\

TuckER & \cellcolor{rank1}{0.3530} & \cellcolor{rank1}{0.4608} & \cellcolor{rank1}{0.7379} & \cellcolor{rank1}{0.9479} & \cellcolor{rank1}{0.9884}
 & \cellcolor{rank5}{0.4625} & \cellcolor{rank5}{0.4999} & \cellcolor{rank6}{0.6368} & \cellcolor{rank5}{0.7961} & \cellcolor{rank5}{0.8583}
 & \cellcolor{rank1}{0.5562} & \cellcolor{rank1}{0.6246} & \cellcolor{rank1}{0.8144} & \cellcolor{rank4}{0.9550} & \cellcolor{rank4}{0.9806}
 
        \\

pLogicNet & \cellcolor{rank2}{0.3241} & \cellcolor{rank2}{0.4311} & \cellcolor{rank3}{0.7149} & \cellcolor{rank3}{0.9386} & \cellcolor{rank3}{0.9847}
 & \cellcolor{rank1}{0.4957} & \cellcolor{rank1}{0.5672} & \cellcolor{rank1}{0.7593} & \cellcolor{rank1}{0.9002} & \cellcolor{rank2}{0.9292}
 & \cellcolor{rank4}{0.4535} & \cellcolor{rank4}{0.5467} & \cellcolor{rank4}{0.7927} & \cellcolor{rank2}{0.9576} & \cellcolor{rank1}{0.9834}
 
        \\

RNNLogic & \cellcolor{rank3}{0.3197} & \cellcolor{rank5}{0.4218} & \cellcolor{rank5}{0.6969} & \cellcolor{rank6}{0.9213} & \cellcolor{rank6}{0.9723}
 & \cellcolor{rank2}{0.4807} & \cellcolor{rank2}{0.5225} & \cellcolor{rank5}{0.6567} & \cellcolor{rank6}{0.7929} & \cellcolor{rank6}{0.8401}

& OOM & OOM & OOM & OOM & OOM 
        \\
% \midrule
% $max(\Delta)$ & 0.0443 & 0.0508 & 0.0469 & 0.0266 & 0.0161
% & 0.0773 & 0.0759 & 0.1225 & 0.1073 & 0.0945
% & 0.1158 & 0.1022 & 0.0676 & 0.0321 & 0.0187 \\

\bottomrule
\end{tabular}
\begin{tabular}{l ccccc | ccccc | ccccc}
\multirow{3}{*}{Models} 
    & \multicolumn{5}{c}{Family} 
    & \multicolumn{5}{c}{UMLS}
    & \multicolumn{5}{c}{Kinship}\\
\cmidrule(lr){2-16} 
    & $\alpha$=1.0 & $\alpha$=0.5 & $\alpha$=0.0 & $\alpha$=-0.5 & $\alpha$=-1.0   
    & $\alpha$=1.0 & $\alpha$=0.5 & $\alpha$=0.0 & $\alpha$=-0.5 & $\alpha$=-1.0    
    & $\alpha$=1.0 & $\alpha$=0.5 & $\alpha$=0.0 & $\alpha$=-0.5 & $\alpha$=-1.0   \\
\midrule

RotatE  & \cellcolor{rank4}{0.9487} & \cellcolor{rank4}{0.9645} & \cellcolor{rank5}{0.9834} & \cellcolor{rank4}{0.9922} & \cellcolor{rank4}{0.9939}
 & \cellcolor{rank4}{0.7784} & \cellcolor{rank4}{0.8368} & \cellcolor{rank4}{0.9054} & \cellcolor{rank3}{0.9552} & \cellcolor{rank3}{0.9771}
 & \cellcolor{rank4}{0.6336} & \cellcolor{rank4}{0.7236} & \cellcolor{rank4}{0.8372} & \cellcolor{rank4}{0.9276} & \cellcolor{rank3}{0.9719}
    
        \\

ComplEx  & \cellcolor{rank5}{0.9434} & \cellcolor{rank5}{0.9619} & \cellcolor{rank4}{0.9837} & \cellcolor{rank1}{0.9942} & \cellcolor{rank1}{0.9967}
 & \cellcolor{rank5}{0.6046} & \cellcolor{rank5}{0.6842} & \cellcolor{rank5}{0.7936} & \cellcolor{rank5}{0.8879} & \cellcolor{rank5}{0.9386}
 & \cellcolor{rank5}{0.5394} & \cellcolor{rank5}{0.6462} & \cellcolor{rank5}{0.7857} & \cellcolor{rank5}{0.9008} & \cellcolor{rank5}{0.9594}

        \\

HousE  & \cellcolor{rank2}{0.9757} & \cellcolor{rank2}{0.9818} & \cellcolor{rank2}{0.9893} & \cellcolor{rank3}{0.9938} & \cellcolor{rank3}{0.9953}
 & \cellcolor{rank2}{0.8354} & \cellcolor{rank2}{0.8775} & \cellcolor{rank2}{0.9271} & \cellcolor{rank2}{0.9637} & \cellcolor{rank2}{0.9802}
 & \cellcolor{rank2}{0.6808} & \cellcolor{rank2}{0.7620} & \cellcolor{rank2}{0.8621} & \cellcolor{rank2}{0.9400} & \cellcolor{rank1}{0.9770}

        \\

TuckER  & \cellcolor{rank1}{0.9784} & \cellcolor{rank1}{0.9832} & \cellcolor{rank1}{0.9895} & \cellcolor{rank2}{0.9939} & \cellcolor{rank2}{0.9959}
 & \cellcolor{rank1}{0.8439} & \cellcolor{rank1}{0.8849} & \cellcolor{rank1}{0.9323} & \cellcolor{rank1}{0.9669} & \cellcolor{rank1}{0.9825}
 & \cellcolor{rank1}{0.7067} & \cellcolor{rank1}{0.7808} & \cellcolor{rank1}{0.8716} & \cellcolor{rank1}{0.9421} & \cellcolor{rank2}{0.9759}

        \\

pLogicNet & \cellcolor{rank6}{0.8119} & \cellcolor{rank6}{0.8600} & \cellcolor{rank6}{0.9355} & \cellcolor{rank6}{0.9837} & \cellcolor{rank5}{0.9939}
 & \cellcolor{rank6}{0.5781} & \cellcolor{rank6}{0.6608} & \cellcolor{rank6}{0.7748} & \cellcolor{rank6}{0.8764} & \cellcolor{rank6}{0.9346}
 & \cellcolor{rank6}{0.3771} & \cellcolor{rank6}{0.4501} & \cellcolor{rank6}{0.5628} & \cellcolor{rank6}{0.6803} & \cellcolor{rank6}{0.7641}

\\

RNNLogic  & \cellcolor{rank3}{0.9737} & \cellcolor{rank3}{0.9781} & \cellcolor{rank3}{0.9839} & \cellcolor{rank5}{0.9882} & \cellcolor{rank6}{0.9901}
 & \cellcolor{rank3}{0.7942} & \cellcolor{rank3}{0.8455} & \cellcolor{rank3}{0.9068} & \cellcolor{rank4}{0.9526} & \cellcolor{rank4}{0.9739}
 & \cellcolor{rank3}{0.6417} & \cellcolor{rank3}{0.7289} & \cellcolor{rank3}{0.8394} & \cellcolor{rank3}{0.9277} & \cellcolor{rank4}{0.9712}

        \\
        % \midrule
% $\Delta$ & - & - & - & - & 
% & - & - & - & - & 
% & - & - & - & - &  \\
% \bottomrule

\bottomrule
\end{tabular}
\label{table:rq-1}
\end{table}

\subsection{RQ1. Effects of predictive sharpness}\label{sec:eval-rq1}
In this experiment, we measure the accuracy of six KGC models with the varying sharpness control factor $\alpha$, while fixing the level of popularity-bias robustness ($\beta=0$).
As shown in \ref{table:rq-1}, the results clearly demonstrate that the accuracy of KGC models are highly sensitive to the level of predictive sharpness controlled by $\alpha$.
As $\alpha$ varies, not only the absolute evaluation scores but also the relative rankings of models substantially change across datasets, indicating that different levels of predictive sharpness may favor different KGC models.
This observation suggests that selecting an inappropriate level of predictive sharpness for a target application (e.g, drug discovery or recommendation) may lead to the selection of a suboptimal model.
For instance, on FB15k-237, ComplEx is ranked only 4th when $\alpha=1.0$ (i.e., a high-sharpness setting similar to MRR), 
but its ranking progressively improves to 3rd at $\alpha=0.5$ and further to 2nd when $\alpha=0.0$.
Similarly, on YAGO3-10, ComplEx eventually becomes the best-performing model when $\alpha=-0.5$, despite being ranked only 5th under $\alpha=1.0$.

\begin{figure}[t]
    \centering
    \begin{tabular}{cc}
        \includegraphics[width=0.43\linewidth]{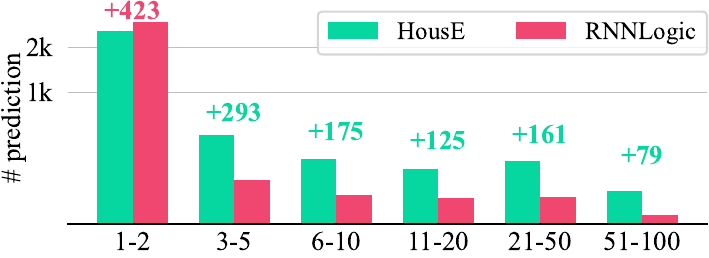} & \includegraphics[width=0.43\linewidth]{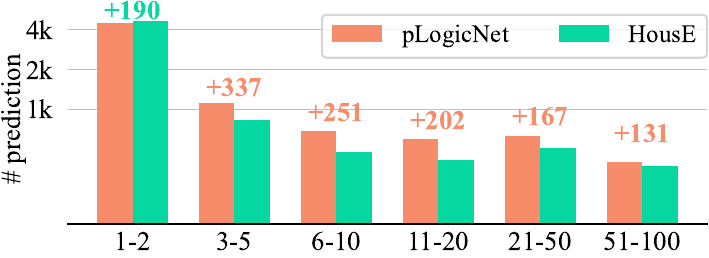} \\
    \end{tabular}
    \vspace{-4mm}
    \caption{Breakdown of predictions of two KGC models according to their ranks on the WN18RR and YAGO3-10 datasets.}
    \label{fig:rq1-case}
    \vspace{-3mm}
\end{figure}

As $\alpha$ decreases, {\method} imposes weaker penalties on non-top-ranked predictions, allowing moderately accurate predictions to contribute more positively to the final score.
Thus, models that consistently produce relatively high ranks tend to improve their relative rankings under low predictive sharpness settings.
In contrast, models that rely on a relatively small number of highly accurate top-ranked predictions tend to lose their advantage as $\alpha$ decreases.
For example, on WN18RR, HousE improves from 6th at $\alpha=1.0$ to 1st at $\alpha=-1.0$, whereas RNNLogic shows the opposite trend.
This suggests that HousE produces more stable rank predictions across broader rank ranges, while RNNLogic relies more heavily on top-ranked predictions.
A similar pattern is observed on YAGO3-10, where pLogicNet becomes increasingly competitive under lower predictive sharpness settings despite weaker performance under high predictive sharpness settings.
Specifically, as shown in Figure~\ref{fig:rq1-case}, although RNNLogic achieves slightly more rank-1 and rank-2 predictions, HousE substantially outperforms RNNLogic across ranks 3--100.
Consequently, RNNLogic is favored under high predictive sharpness settings, while HousE becomes increasingly competitive as $\alpha$ decreases.
A similar pattern is also observed between pLogicNet and HousE on YAGO3-10.
Therefore, different applications may require different levels of predictive sharpness, and failing to reflect this perspective may lead to suboptimal model selection.

\vspace{1mm}
\noindent
\textbf{Finding (1)}.
Existing evaluation metrics for KGC models (e.g., MRR) implicitly assume high predictive sharpness, favoring KGC models producing top-ranked predictions while underestimating models with more stable ranking quality across broader rank ranges.
This suggests that a single fixed metric may be insufficient across diverse application scenarios with different evaluation requirements.

\begin{table}[ht]
\centering
% \small % or 
\footnotesize
\caption{The accuracy of KGC Models evaluated by {\method} with varying levels of popularity-bias robustness $\beta$. The \textcolor[HTML]{F8A602}{gold}, \textcolor[HTML]{A5A5B1}{silver}, and \textcolor[HTML]{D36A49}{bronze} indicate the best, the second, and the third best results, respectively.}
\vspace{-3mm}
\setlength{\tabcolsep}{3.0pt}
\begin{tabular}{l ccccc | ccccc | ccccc}
\toprule
\toprule
\multirow{3}{*}{Models} 
    & \multicolumn{5}{c}{FB15k-237} 
    & \multicolumn{5}{c}{WN18RR}
    & \multicolumn{5}{c}{YAGO3-10}\\
\cmidrule(lr){2-16} 
    & $\beta$=0.0 & $\beta$=0.2 & $\beta$=0.4 & $\beta$=0.6 & $\beta$=0.8   
    & $\beta$=0.0 & $\beta$=0.2 & $\beta$=0.4 & $\beta$=0.6 & $\beta$=0.8    
    & $\beta$=0.0 & $\beta$=0.2 & $\beta$=0.4 & $\beta$=0.6 & $\beta$=0.8    \\
\midrule

RotatE & \cellcolor{rank5}{0.3140} & \cellcolor{rank4}{0.1747} & \cellcolor{rank5}{0.1043} & \cellcolor{rank5}{0.0866} & \cellcolor{rank5}{0.0820}
 & \cellcolor{rank4}{0.4641} & \cellcolor{rank3}{0.3326} & \cellcolor{rank2}{0.2097} & \cellcolor{rank2}{0.1462} & \cellcolor{rank2}{0.1203}
 & \cellcolor{rank2}{0.4743} & \cellcolor{rank2}{0.3525} & \cellcolor{rank2}{0.2207} & \cellcolor{rank3}{0.1249} & \cellcolor{rank3}{0.0791}
 
        \\

ComplEx & \cellcolor{rank4}{0.3146} & \cellcolor{rank3}{0.1801} & \cellcolor{rank3}{0.1097} & \cellcolor{rank4}{0.0897} & \cellcolor{rank4}{0.0821}
 & \cellcolor{rank3}{0.4754} & \cellcolor{rank2}{0.3330} & \cellcolor{rank3}{0.1998} & \cellcolor{rank4}{0.1309} & \cellcolor{rank4}{0.1029}
 & \cellcolor{rank5}{0.4496} & \cellcolor{rank3}{0.3476} & \cellcolor{rank3}{0.2188} & \cellcolor{rank4}{0.1241} & \cellcolor{rank4}{0.0789}
 
        \\

HousE & \cellcolor{rank6}{0.3073} & \cellcolor{rank6}{0.1651} & \cellcolor{rank6}{0.0948} & \cellcolor{rank6}{0.0775} & \cellcolor{rank6}{0.0731}
 & \cellcolor{rank6}{0.4184} & \cellcolor{rank6}{0.3017} & \cellcolor{rank4}{0.1941} & \cellcolor{rank3}{0.1383} & \cellcolor{rank3}{0.1153}
 & \cellcolor{rank3}{0.4538} & \cellcolor{rank5}{0.3268} & \cellcolor{rank5}{0.1828} & \cellcolor{rank5}{0.0778} & \cellcolor{rank5}{0.0297}
 
        \\

TuckER & \cellcolor{rank1}{0.3530} & \cellcolor{rank1}{0.2033} & \cellcolor{rank1}{0.1286} & \cellcolor{rank2}{0.1073} & \cellcolor{rank2}{0.0980}
 & \cellcolor{rank5}{0.4625} & \cellcolor{rank5}{0.3125} & \cellcolor{rank6}{0.1751} & \cellcolor{rank6}{0.1058} & \cellcolor{rank6}{0.0790}
 & \cellcolor{rank1}{0.5654} & \cellcolor{rank1}{0.4260} & \cellcolor{rank1}{0.2675} & \cellcolor{rank1}{0.1505} & \cellcolor{rank1}{0.0947}
 
        \\

pLogicNet & \cellcolor{rank2}{0.3241} & \cellcolor{rank2}{0.1912} & \cellcolor{rank2}{0.1244} & \cellcolor{rank1}{0.1103} & \cellcolor{rank1}{0.1097}
 & \cellcolor{rank1}{0.4957} & \cellcolor{rank1}{0.3478} & \cellcolor{rank1}{0.2221} & \cellcolor{rank1}{0.1601} & \cellcolor{rank1}{0.1354}
 & \cellcolor{rank4}{0.4535} & \cellcolor{rank4}{0.3391} & \cellcolor{rank4}{0.2167} & \cellcolor{rank2}{0.1279} & \cellcolor{rank2}{0.0861}
 
        \\

RNNLogic & \cellcolor{rank3}{0.3197} & \cellcolor{rank5}{0.1745} & \cellcolor{rank4}{0.1058} & \cellcolor{rank3}{0.0901} & \cellcolor{rank3}{0.0862}
 & \cellcolor{rank2}{0.4807} & \cellcolor{rank4}{0.3243} & \cellcolor{rank5}{0.1820} & \cellcolor{rank5}{0.1099} & \cellcolor{rank5}{0.0817}
 
& OOM & OOM & OOM & OOM & OOM 
        \\
\bottomrule
\end{tabular}

\begin{tabular}{l ccccc | ccccc | ccccc}
\multirow{3}{*}{Models} 
    & \multicolumn{5}{c}{Family} 
    & \multicolumn{5}{c}{UMLS}
    & \multicolumn{5}{c}{Kinship}\\
\cmidrule(lr){2-16} 
     & $\beta$=0.0 & $\beta$=0.2 & $\beta$=0.4 & $\beta$=0.6 & $\beta$=0.8   
    & $\beta$=0.0 & $\beta$=0.2 & $\beta$=0.4 & $\beta$=0.6 & $\beta$=0.8    
    & $\beta$=0.0 & $\beta$=0.2 & $\beta$=0.4 & $\beta$=0.6 & $\beta$=0.8    \\
\midrule

RotatE & \cellcolor{rank4}{0.9487} & \cellcolor{rank5}{0.9405} & \cellcolor{rank5}{0.9259} & \cellcolor{rank5}{0.9035} & \cellcolor{rank5}{0.8745}
 & \cellcolor{rank4}{0.7784} & \cellcolor{rank4}{0.7537} & \cellcolor{rank3}{0.7077} & \cellcolor{rank3}{0.6310} & \cellcolor{rank4}{0.5254}
 & \cellcolor{rank4}{0.6336} & \cellcolor{rank3}{0.6273} & \cellcolor{rank3}{0.6199} & \cellcolor{rank3}{0.6113} & \cellcolor{rank3}{0.6016}

        \\

ComplEx & \cellcolor{rank5}{0.9434} & \cellcolor{rank4}{0.9419} & \cellcolor{rank4}{0.9344} & \cellcolor{rank3}{0.9185} & \cellcolor{rank3}{0.8944}
 & \cellcolor{rank5}{0.6046} & \cellcolor{rank5}{0.6126} & \cellcolor{rank5}{0.6092} & \cellcolor{rank5}{0.5888} & \cellcolor{rank3}{0.5512}
 & \cellcolor{rank5}{0.5394} & \cellcolor{rank5}{0.5354} & \cellcolor{rank5}{0.5305} & \cellcolor{rank5}{0.5247} & \cellcolor{rank5}{0.5178}

        \\

HousE & \cellcolor{rank2}{0.9757} & \cellcolor{rank2}{0.9672} & \cellcolor{rank2}{0.9513} & \cellcolor{rank2}{0.9266} & \cellcolor{rank2}{0.8950}
 & \cellcolor{rank2}{0.8354} & \cellcolor{rank1}{0.8128} & \cellcolor{rank1}{0.7677} & \cellcolor{rank1}{0.6903} & \cellcolor{rank2}{0.5820}
 & \cellcolor{rank2}{0.6808} & \cellcolor{rank2}{0.6709} & \cellcolor{rank2}{0.6594} & \cellcolor{rank2}{0.6464} & \cellcolor{rank1}{0.6320}

        \\

TuckER & \cellcolor{rank1}{0.9784} & \cellcolor{rank1}{0.9700} & \cellcolor{rank1}{0.9548} & \cellcolor{rank1}{0.9318} & \cellcolor{rank1}{0.9026}
 & \cellcolor{rank1}{0.8439} & \cellcolor{rank2}{0.8111} & \cellcolor{rank2}{0.7597} & \cellcolor{rank2}{0.6878} & \cellcolor{rank1}{0.6054}
 & \cellcolor{rank1}{0.7067} & \cellcolor{rank1}{0.6912} & \cellcolor{rank1}{0.6733} & \cellcolor{rank1}{0.6531} & \cellcolor{rank2}{0.6308}

        \\
pLogicNet & \cellcolor{rank6}{0.8119} & \cellcolor{rank6}{0.7913} & \cellcolor{rank6}{0.7662} & \cellcolor{rank6}{0.7393} & \cellcolor{rank6}{0.7131}
 & \cellcolor{rank6}{0.5781} & \cellcolor{rank6}{0.5590} & \cellcolor{rank6}{0.5258} & \cellcolor{rank6}{0.4774} & \cellcolor{rank6}{0.4202}
 & \cellcolor{rank6}{0.3771} & \cellcolor{rank6}{0.3795} & \cellcolor{rank6}{0.3817} & \cellcolor{rank6}{0.3834} & \cellcolor{rank6}{0.3847}

\\

RNNLogic & \cellcolor{rank3}{0.9737} & \cellcolor{rank3}{0.9612} & \cellcolor{rank3}{0.9401} & \cellcolor{rank4}{0.9102} & \cellcolor{rank4}{0.8748}
 & \cellcolor{rank3}{0.7942} & \cellcolor{rank3}{0.7608} & \cellcolor{rank4}{0.7063} & \cellcolor{rank4}{0.6251} & \cellcolor{rank5}{0.5233}
 & \cellcolor{rank3}{0.6417} & \cellcolor{rank4}{0.6230} & \cellcolor{rank4}{0.6016} & \cellcolor{rank4}{0.5775} & \cellcolor{rank4}{0.5510}

        \\
        
\bottomrule
\bottomrule
\end{tabular}
\label{table:rq2}
\end{table}

\noindent
\subsection{RQ2. Effects of popularity-bias robustness}\label{sec:eval-rq2}
In this experiment, we evaluate six KGC models with varying levels of popularity-bias robustness $\beta$, while fixing the predictive sharpness level ($\alpha=1$).
As shown in Table~\ref{table:rq2}, increasing $\beta$ significantly changes the evaluation behavior across all datasets.
These results indicate that the perceived effectiveness of a KGC model can vary considerably depending on how importantly the evaluation emphasizes robustness to popularity bias (e.g., discovering a new fact or allowing common facts).
Therefore, ignoring popularity bias during evaluation may result in selecting models that perform well mainly on highly popular facts, while overlooking models that better generalize to less-observed new knowledge.

Specifically, as $\beta$ increases, {\method} assigns larger weights, i.e., $\delta(e)$ and $\delta(r|e)$, to low-popularity triples while assigning smaller weights to high-popularity triples.
Consequently, KGC models that rely heavily on predictions on high-popularity triples tend to lose their advantage as $\beta$ increases.
In contrast, models that maintain high-quality predictions on low-popularity triples become increasingly competitive.
For example, on FB15k-237, TuckER is ranked 1st when $\beta=0$, but its ranking decreases under more popularity-robust settings, whereas pLogicNet becomes the best-performing model when $\beta \geq 0.6$.
Similarly, on WN18RR, RNNLogic rapidly drops from 2nd to 5th as $\beta$ increases, while RotatE improves from 4th to 2nd.

To better understand why the relative rankings of KGC models change as $\beta$ increases,
we conduct an in-depth analysis by comparing their prediction behaviors across different popularity levels.
Specifically, we divide test queries into four groups according to their triple popularity weight $w_t$ and sample representative queries from each group.

\vspace{1mm}
\noindent
\textbf{Case study 1: pLogicNet vs TuckER on FB15k-237}.
Figure~\ref{fig:rq2-case}(a) shows the prediction results of pLogicNet and TuckER across four popularity groups.
For the most popular group (e.g., Top 0.34\%), TuckER achieves a more accurate prediction (rank 2) than pLogicNet (rank 5), 
indicating that TuckER performs particularly well on highly frequent facts.
However, as popularity decreases, pLogicNet consistently maintains relatively strong prediction quality, 
achieving ranks 6, 7, and 20 across the remaining groups.
In contrast, TuckER’s prediction quality substantially deteriorates, producing much lower-ranked predictions (ranks 15, 18, and 50).
These results explain why pLogicNet becomes increasingly favored under higher popularity-bias robustness settings.
We further analyze the models from the perspective of entity-conditioned relation popularity $\delta(r|e)$.
As shown in Figure~\ref{fig:rq2-case}(a) (below pie charts), 
TuckER performs extremely well on highly frequent entity-relation patterns.
For example, in the query $(\textit{Brunei}, \textit{organization}, ?)$, the relation \textit{organization} accounts for 91\% of all relations associated with the entity, 
and TuckER successfully predicts the correct answer at rank 1.
However, on less frequent entity-relation patterns, such as $(\textit{Tenor saxophone}, role, ?)$ and $(\textit{Tak Fujimoto}, \textit{cinematography}, ?)$, 
pLogicNet substantially outperforms TuckER.
These observations indicate that TuckER relies more heavily on highly popular entity-relation patterns, whereas pLogicNet demonstrates stronger robustness to sparse and less-observed facts.

\begin{figure}
    \centering
    \begin{subfigure}[t]{0.48\textwidth}
    \centering
    \begin{subfigure}{0.24\linewidth}
        \includegraphics[width=\linewidth]{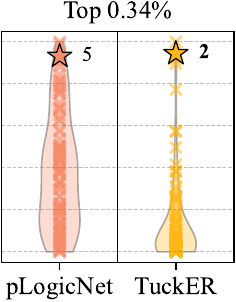}
    \end{subfigure}
    \begin{subfigure}{0.24\linewidth}
        \includegraphics[width=\linewidth]{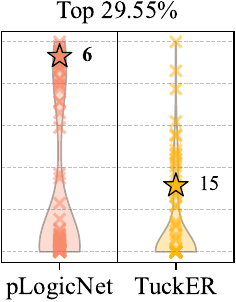}
    \end{subfigure}
    \begin{subfigure}{0.24\linewidth}
        \includegraphics[width=\linewidth]{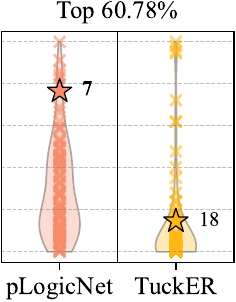}
    \end{subfigure}
    \begin{subfigure}{0.24\linewidth}
        \includegraphics[width=\linewidth]{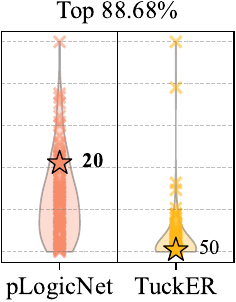}
    \end{subfigure}
    \begin{subfigure}{\linewidth}
        \includegraphics[width=\linewidth]{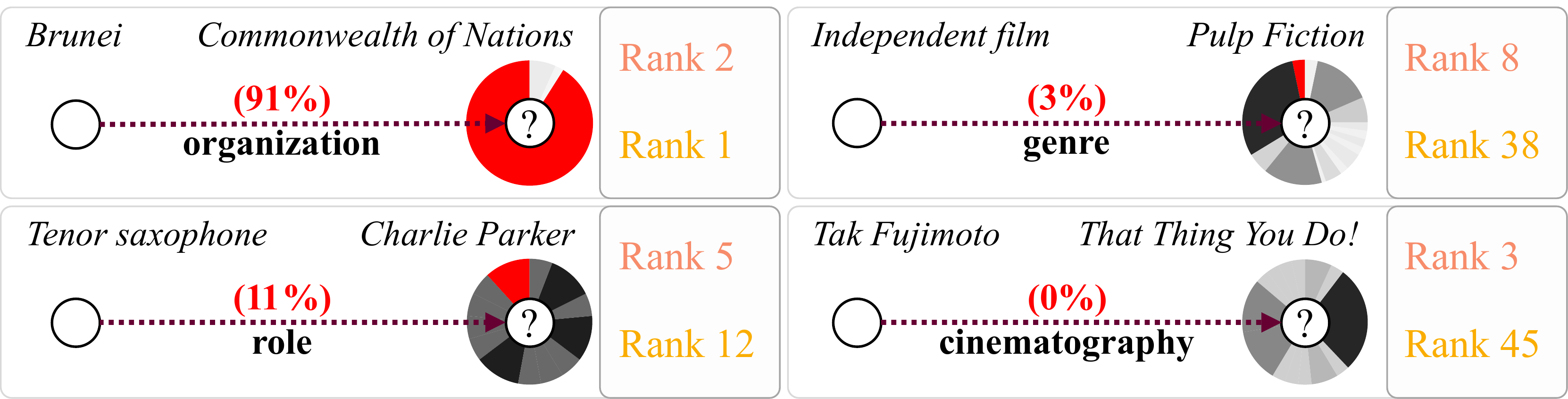}
    \end{subfigure}
    \caption{Case study 1: \textcolor{plogicnet}{pLogicNet} vs \textcolor{tucker}{TuckER} on FB15k-237}
    \end{subfigure}
    \hfill
    \begin{subfigure}[t]{0.48\textwidth}
    \centering
    \begin{subfigure}{0.24\linewidth}
        \includegraphics[width=\linewidth]{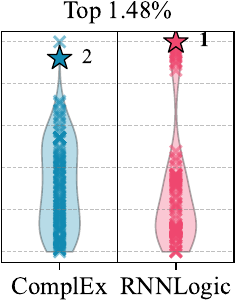}
    \end{subfigure}
    \begin{subfigure}{0.24\linewidth}
        \includegraphics[width=\linewidth]{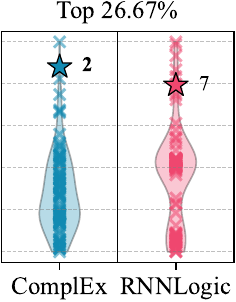}
    \end{subfigure}
    \begin{subfigure}{0.24\linewidth}
        \includegraphics[width=\linewidth]{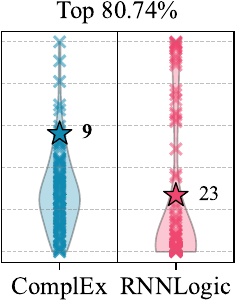}
    \end{subfigure}
    \begin{subfigure}{0.24\linewidth}
        \includegraphics[width=\linewidth]{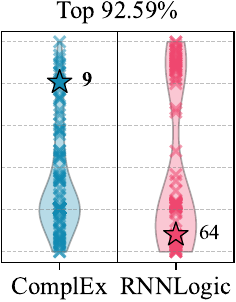}
    \end{subfigure}
    \begin{subfigure}{\linewidth}
        \includegraphics[width=\linewidth]{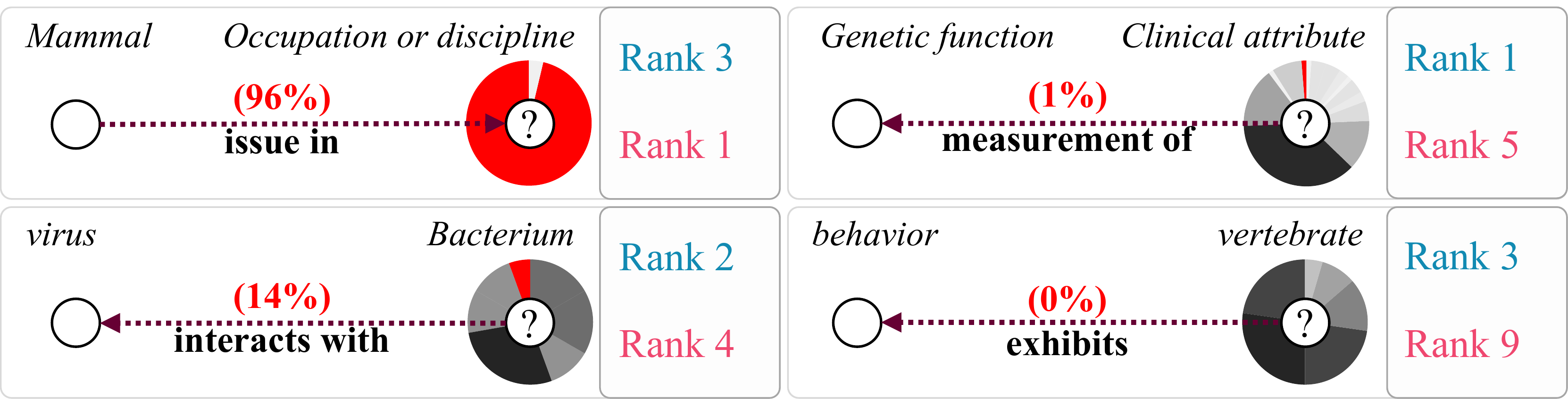}
    \end{subfigure}
    \caption{Case study 2: \textcolor{complex}{ComplEx} vs \textcolor{rnnlogic}{RNNLogic} on UMLS}
    \end{subfigure}
    \vspace{-3mm}
    \caption{Case studies on popularity-bias robustness across varying levels of entity popularity. While TuckER and RNNLogic achieve highly accurate predictions on popular facts, their prediction quality substantially degrades on low-popularity queries. In contrast, pLogicNet and ComplEx maintain more stable prediction quality across broader popularity ranges.}
    \label{fig:rq2-case}
    \vspace{-5mm}
\end{figure}

\vspace{1mm}
\noindent
\textbf{Case study 2: ComplEx vs RNNLogic on UMLS}.
We observe a similar tendency on WN18RR when comparing ComplEx and RNNLogic as shown in Figure~\ref{fig:rq2-case}(b).
Although RNNLogic achieves rank 1 on the most popular query group (Top 1.48\%),
its prediction quality rapidly degrades, producing substantially lower-ranked predictions (e.g., ranks 23 and 64) as query popularity decreases.
In contrast, ComplEx maintains consistently strong prediction quality across all popularity groups, achieving ranks 2, 2, 9, and 9 even for low-popularity queries.
This explains why ComplEx becomes increasingly competitive under higher popularity-bias robustness settings.
From the perspective of entity-conditioned relation popularity $\delta(r|e)$,
RNNLogic achieves highly accurate predictions for highly frequent entity-relation patterns, such as $(Mammal, issue in, ?)$ where the relation accounts for 96\% of the entity’s associated relations.
However, on rare entity-relation patterns, such as $(?, measurement of, Genetic function)$ and $(?, exhibits, behavior)$, ComplEx outperforms RNNLogic.

\vspace{1mm}
\noindent
\textbf{Finding (2)}.
Existing metrics implicitly assume uniform importance across all triples, favoring KGC models that rely heavily on highly popular entities and relations (i.e., exhibiting strong popularity bias), while underestimating models robust to less-observed yet important facts. 
As a result, conventional evaluation settings may overestimate the practical effectiveness of KGC models that fail to generalize beyond frequently observed knowledge.

\subsection{RQ3. Comprehensiveness of {\method}}\label{sec:eval-rq3}

\begin{figure}
    \centering
    \footnotesize
    \newcommand{\width}{0.07\linewidth}
    \begin{tabular}{c|cccccc}
        \toprule
        &RotatE&ComplEx&HousE&TuckER&pLogicNet&RNNLogic \\
        \midrule
        \rotatebox{90}{FB15k-237} & \includegraphics[width=\width]{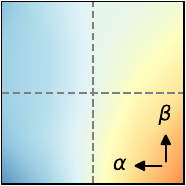} & \includegraphics[width=\width]{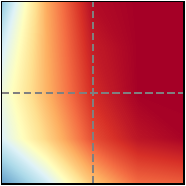} & \includegraphics[width=\width]{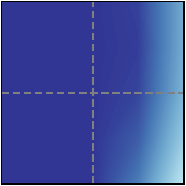} & \includegraphics[width=\width]{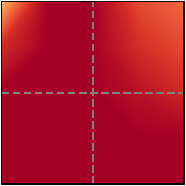} & \includegraphics[width=\width]{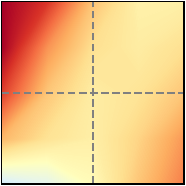} & \includegraphics[width=\width]{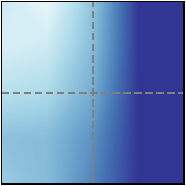} \\
        \midrule
        
        \rotatebox{90}{WN18RR} & \includegraphics[width=\width]{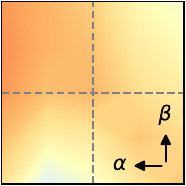} & \includegraphics[width=\width]{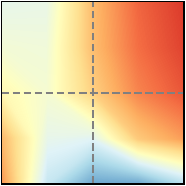} & \includegraphics[width=\width]{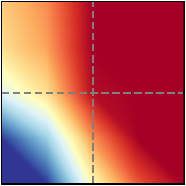} & \includegraphics[width=\width]{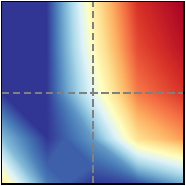} & \includegraphics[width=\width]{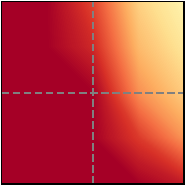} & \includegraphics[width=\width]{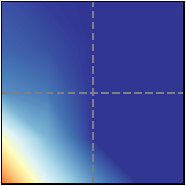} \\
        \midrule
        \rotatebox{90}{YAGO3-10} & \includegraphics[width=\width]{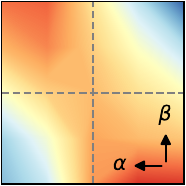} & \includegraphics[width=\width]{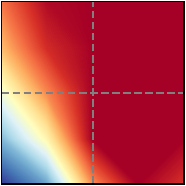} & \includegraphics[width=\width]{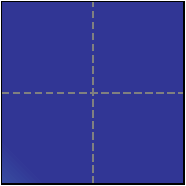} & \includegraphics[width=\width]{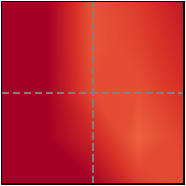} & \includegraphics[width=\width]{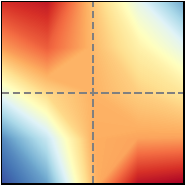} & \includegraphics[width=\width]{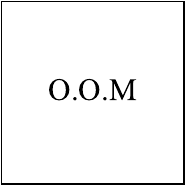} \\
        \midrule
        \rotatebox{90}{Family} & \includegraphics[width=\width]{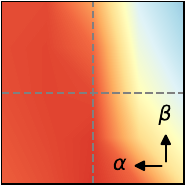} & \includegraphics[width=\width]{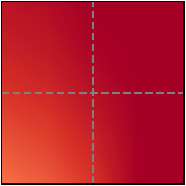} & \includegraphics[width=\width]{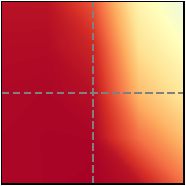} & \includegraphics[width=\width]{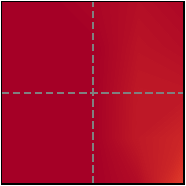} & \includegraphics[width=\width]{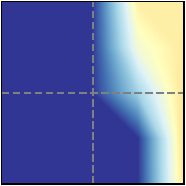} & \includegraphics[width=\width]{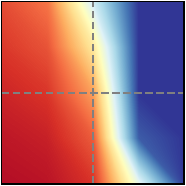} \\
        \midrule
        \rotatebox{90}{UMLS} & \includegraphics[width=\width]{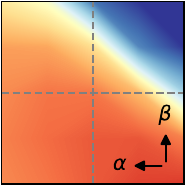} & \includegraphics[width=\width]{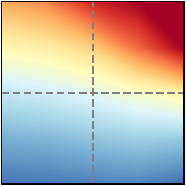} & \includegraphics[width=\width]{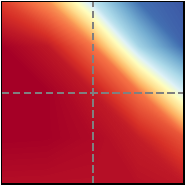} & \includegraphics[width=\width]{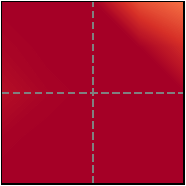} & \includegraphics[width=\width]{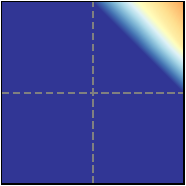} & \includegraphics[width=\width]{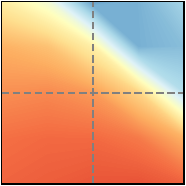} \\
        \midrule
        \rotatebox{90}{Kinship} & \includegraphics[width=\width]{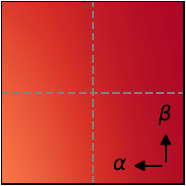} & \includegraphics[width=\width]{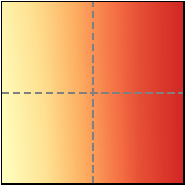} & \includegraphics[width=\width]{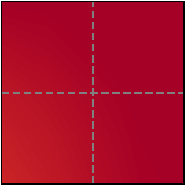} & \includegraphics[width=\width]{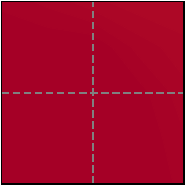} & \includegraphics[width=\width]{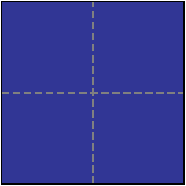} & \includegraphics[width=\width]{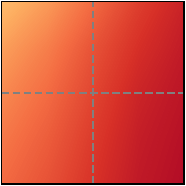} \\
         \bottomrule
    \end{tabular}
    \vspace{-2mm}
    \caption{Comprehensive evaluation of KGC models by {\method} under varying levels of predictive sharpness $\alpha$ and popularity-bias robustness $\beta$. The results demonstrate that model superiority substantially changes across different evaluation perspectives.}
    \vspace{-3mm}
    \label{fig:rq3}
\end{figure}

In this experiment, we investigate how comprehensively {\method} evaluates KGC models across varying levels of predictive sharpness and popularity-bias robustness.
Unlike existing metrics that evaluate KGC models from a single fixed perspective, 
{\method} allows flexible evaluation under diverse combinations of $\alpha$ and $\beta$,
making researchers and practitioners easy to evaluate their KGC models under diverse real-world requirements.

To analyze the relative performance of KGC models across different evaluation perspectives, 
we evaluate each model under multiple $(\alpha,\beta)$ settings.
Since the primary goal of KGC evaluation is to select the best-performing model under a given perspective, 
in this experiment, we apply min-max normalization to the scores of the six comparing models within each $(\alpha,\beta)$ coordinate.
Figure~\ref{fig:rq3} visualizes the performance of each KGC model on each dataset as a heatmap.
In each heatmap, the $x$-axis represents the predictive sharpness factor $\alpha$, ranging from $-1.0$ to $1.0$ with a step size of $0.5$, while the $y$-axis represents the popularity-bias robustness factor $\beta$, ranging from $0.0$ to $0.8$ with a step size of $0.2$.
O.O.M indicates an out-of-memory case.
The results reveal that no single KGC model achieves the best performance across all datasets and evaluation settings.
For example, although TuckER performs best on most datasets under many different evaluation settings, it is outperformed by HousE and pLogicNet on WN18RR and by ComplEx on YAGO3-10.
These results indicate that the preferred KGC model fundamentally depends on the desired evaluation perspective as we claimed.
In addition, model superiority cannot be fully characterized by a single fixed evaluation metric.
In other words, even KGC models that achieve similar performance under conventional metrics often exhibit different performance patterns under varying levels of predictive sharpness and popularity-bias robustness.

\vspace{1mm}
\noindent
\textbf{Finding (3)}.
Existing metrics with a single fixed evaluation perspective fail to fully characterize the performance of KGC models whose relative performance can substantially vary depending on the desired evaluation perspective.
{\method} enables comprehensive evaluation of KGC models across diverse levels of predictive sharpness and popularity-bias robustness, revealing performance characteristics that cannot be captured by existing fixed metrics.

\begin{table}[t]
\footnotesize
\setlength{\tabcolsep}{4pt}
\caption{Different hyperparameters for two KGC models used in the experiment for RQ4.}
\vspace{-3mm}
\newcommand{\boxsize}{0.2}
    \centering
    \begin{tabular}{cclccccccc}
        \toprule
         \# & label & KGC model & Batch size & \# of negative samples & Dimension & gamma & alpha & learning rate & L3 reg. weight \\
         \toprule
             0& \tikz \fill[fill=rotate1] (0,0) rectangle (\boxsize,\boxsize); & RotatE & 1,024 & 128 & 1,000 & 24 & 1.0 & 5e-5 & 0.0 \\
             1& \tikz \fill[fill=rotate2] (0,0) rectangle (\boxsize,\boxsize); & RotatE & 512 & 128 & 500 & 24 & 1.0 & 5e-5 & 0.0 \\
             2& \tikz \fill[fill=rotate3] (0,0) rectangle (\boxsize,\boxsize); & RotatE & 128 & 128 & 200 & 24 & 1.0 & 5e-5 & 0.0 \\
             3& \tikz \fill[fill=rotate4] (0,0) rectangle (\boxsize,\boxsize); & RotatE & 128 & 512 & 500 & 18 & 1.0 & 5e-5 & 0.0 \\
             4& \tikz \fill[fill=rotate5] (0,0) rectangle (\boxsize,\boxsize); & RotatE & 128 & 128 & 1,000 & 12 & 0.5 & 5e-5 & 0.0 \\
             5& \tikz \fill[fill=complex1] (0,0) rectangle (\boxsize,\boxsize); & ComplEx & 1,024 & 128 & 1,000 & 500 & 1.0 & 1e-3 & 1e-5 \\
             6& \tikz \fill[fill=complex2] (0,0) rectangle (\boxsize,\boxsize); & ComplEx & 512 & 128 & 1,000 & 200 & 1.0 & 1e-4 & 1e-5 \\
             7& \tikz \fill[fill=complex3] (0,0) rectangle (\boxsize,\boxsize); & ComplEx & 128 & 128 & 1,000 & 500 & 0.5 & 5e-3 & 1e-5 \\
             8& \tikz \fill[fill=complex4] (0,0) rectangle (\boxsize,\boxsize); & ComplEx & 256 & 128 & 1,000 & 200 & 0.5 & 1e-3 & 1e-5 \\
             9& \tikz \fill[fill=complex5] (0,0) rectangle (\boxsize,\boxsize); & ComplEx & 512 & 256 & 1,000 & 500 & 1.0 & 1e-3 & 1e-5 \\
         \bottomrule
    \end{tabular}
    \label{tab:modelconfig}
\end{table}

\begin{figure}[t]
\centering
\setlength{\tabcolsep}{2.0pt}
\newcommand{\boxsize}{0.25}
\newcommand{\ftgap}{0.2em}
\newcommand{\cellgap}{0.5em}
% ---------- row 1 ----------
\begin{minipage}[t]{0.24\textwidth}
    \centering
    \includegraphics[width=\linewidth]{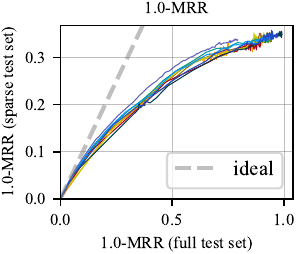}

    \vspace{\ftgap}

    \includegraphics[width=\linewidth]{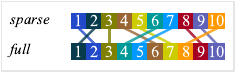}
\end{minipage}\hfill
\begin{minipage}[t]{0.24\textwidth}
    \centering
    \includegraphics[width=\linewidth]{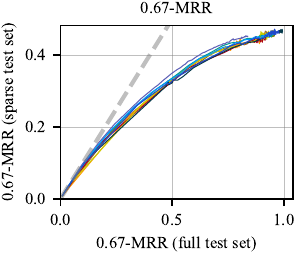}

    \vspace{\ftgap}
    \includegraphics[width=\linewidth]{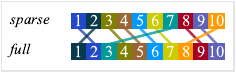}
\end{minipage}\hfill
\begin{minipage}[t]{0.24\textwidth}
    \centering
    \includegraphics[width=\linewidth]{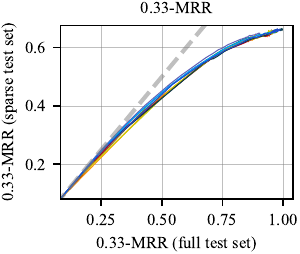}

    \vspace{\ftgap}

    \includegraphics[width=\linewidth]{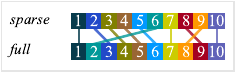}
\end{minipage}\hfill
\begin{minipage}[t]{0.24\textwidth}
    \centering
    \includegraphics[width=\linewidth]{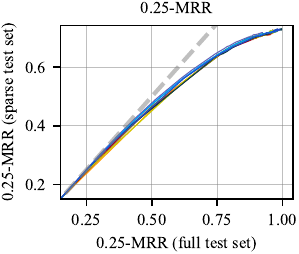}

    \vspace{\ftgap}

    \includegraphics[width=\linewidth]{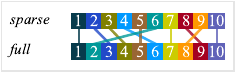}
\end{minipage}

\hrule

\vspace{\cellgap}

% ---------- row 2 ----------
\begin{minipage}[t]{0.24\textwidth}
    \centering
    \includegraphics[width=\linewidth]{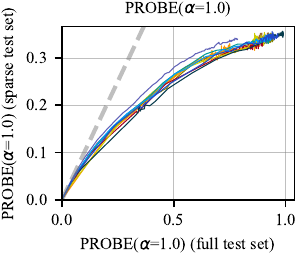}

    \vspace{\ftgap}

    \includegraphics[width=\linewidth]{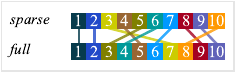}
\end{minipage}\hfill
\begin{minipage}[t]{0.24\textwidth}
    \centering
    \includegraphics[width=\linewidth]{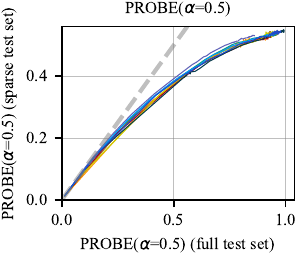}

    \vspace{\ftgap}

    \includegraphics[width=\linewidth]{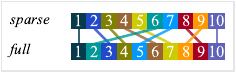}
\end{minipage}\hfill
\begin{minipage}[t]{0.24\textwidth}
    \centering
    \includegraphics[width=\linewidth]{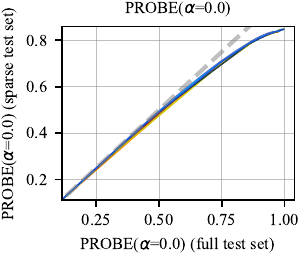}

    \vspace{\ftgap}

    \includegraphics[width=\linewidth]{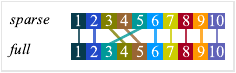}
\end{minipage}\hfill
\begin{minipage}[t]{0.24\textwidth}
    \centering
    \includegraphics[width=\linewidth]{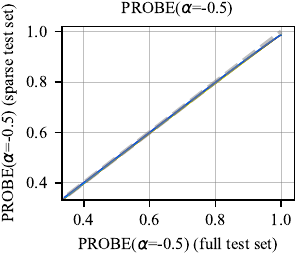}

    \vspace{\ftgap}

    \includegraphics[width=\linewidth]{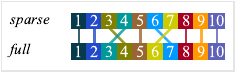}
\end{minipage}

\caption{Comparison of intrinsic model performance under CWA ($x$-axis) and observed performance under OWA ($y$-axis). A reliable evaluation metric should preserve the relative ranking of models and assign similar scores to models with comparable intrinsic performance under CWA, even when evaluated under OWA.}
\label{fig:rq4}
\end{figure}

\noindent
\subsection{RQ4. Evaluation Reliability under the Open-World Assumption}\label{sec:eval-rq4}
In this experiment, we evaluate how reliably {\method} reflects the intrinsic performance of KGC models under the open-world assumption (OWA).
To properly assess intrinsic model performance under OWA, 
the corresponding performance under the closed-world assumption (CWA) must be known. 
However, in real-world KGs, obtaining all true facts is practically impossible.
Therefore, we use the artificially constructed family-tree KG introduced in~\cite{yang2022rethinking}, where all facts are fully observable under CWA. 
The KG consists of 6,004 entities, 23 relations, and 192,532 triples.
From this closed-world KG, OWA settings can be simulated by intentionally removing triples,
enabling direct comparison between the intrinsic performance of KGC models under CWA and their observed evaluation results under OWA.
Specifically, the KG dataset consists of a training set and two test sets.
\textbf{(1) Full test set}: all facts from the complete family-tree KG are used, enabling evaluation under CWA. 
\textbf{(2) Sparse test set}: only 75\% of the facts are used, simulating evaluation under OWA.

We select RotatE and ComplEx as the base KGC models and use five model variants with different hyperparameters for each KGC model as described in Table~\ref{tab:modelconfig}.
Each model is trained on the training set and evaluated on both the full test set and sparse test set by using three evaluation metrics, including p-MRR, Hits@K, and {\method}. 
For p-MRR, we use four values for the hyperparameter $p$, 1.0, 0.67, 0.33, and 0.25, following the settings used in the original work~\cite{yang2022rethinking}.
Figure~\ref{fig:rq4} shows the results, where the $x$-axis represents the model performance on the full test set, i.e., the intrinsic performance under the closed-world assumption (CWA), while the $y$-axis represents the observed performance on the sparse test set under the open-world assumption (OWA).
For Hits@K and p-MRR, models with the same intrinsic performance under CWA are often evaluated quite differently under OWA, and the relative rankings of models substantially change between the two settings.
In contrast, {\method} consistently assigns similar evaluation scores to models with the same CWA performance, while well-preserving their relative rankings under OWA.

\vspace{1mm}
\noindent
\textbf{Finding (4)}.
Existing metrics often fail to reliably reflect the intrinsic performance of KGC models under the open-world assumption, 
as their evaluation results can substantially vary depending on the incompleteness of observable facts.
In contrast, {\method} more consistently preserves the relative performance and rankings of KGC models between CWA and OWA settings, enabling more reliable evaluation under incomplete knowledge graphs.

\section{Conclusion}\label{sec:con}
In this paper, we point out that existing rank-based metrics for KGC evaluation overlook two important perspectives: (\textbf{P1}) predictive sharpness and (\textbf{P2}) popularity-bias robustness.
To address these limitations, we propose \textbf{{\method}}, a generalized KGC evaluation framework consisting of a rank transformer (RT) and a rank aggregator (RA), which flexibly control predictive sharpness and popularity-bias robustness, respectively.
We theoretically analyze {\method} by defining six key properties for reliable KGC evaluation and prove that {\method} satisfies all of them, unlike existing metrics.
In particular, we show that {\method} more consistently preserves the intrinsic performance and relative rankings of KGC models under the open-world assumption.
Extensive experiments with six KGC models on six real-world KGs demonstrate that existing rank-based metrics implicitly assume fixed evaluation perspectives, which can over- or under-estimate KGC models depending on their ranking behavior and robustness to popularity bias, while {\method} enables more comprehensive, flexible, and reliable evaluation across diverse evaluation perspectives.

In future work, we plan to further extend {\method} to a broader range of KGC models and real-world KGs.
We also plan to develop differentiable extensions of {\method} that can be directly incorporated into the training objective, 
enabling KGC models to learn application-specific ranking behaviors and robustness requirements in an end-to-end manner.

% \section{GenAI Usage Disclosur}\label{sec-genai}
% In accordance with the ACM authorship policy, we disclose the usage of generative AI tools (e.g., ChatGPT) as follows.
% \begin{itemize}
%     \item \textbf{GenAI usage in writing}: ChatGPT was only used to review grammatical consistency during the writing the manuscript.
%     \item \textbf{GenAI usage in data processing}: ChatGPT was only used for generating code for drawing figures (e.g., matplotlib.pyplot).
%     \item \textbf{Author responsibility}: All uses of GenAI were limited to assistant roles. We conducted final verification and refinement of all GenAI generated results, analyses, and textual content.
% \end{itemize}

% \begin{acks}
% \end{acks}


\bibliographystyle{ACM-Reference-Format}
\bibliography{bibliography.bib}

@inproceedings{ko2023khan,
  title={KHAN: Knowledge-Aware Hierarchical Attention Networks for Accurate Political Stance Prediction},
  author={Ko, Yunyong and Ryu, Seongeun and Han, Soeun and Jeon,Youngseung and Kim, Jaehoon and Park, Sohyun and Han, Kyungsik Tong, Hanghang and Kim., Sang-Wook},
  booktitle={Proceedings of the ACM Web Conference (WWW)},
  pages={1572--1583},
  year={2023},
  isbn = {9781450394161},
  publisher = {Association for Computing Machinery (ACM)},
  doi = {10.1145/3543507.3583300},
  location = {Austin, TX, USA},
}

@inproceedings{zhang2022nclass-naacl,
  title={KCD: Knowledge Walks and Textual Cues Enhanced Political Perspective Detection in News Media},
  author={Zhang, Wenqian and Feng, Shangbin and Chen, Zilong and Lei, Zhenyu and Li, Jundong and Luo, Minnan},
  booktitle={Proceedings of the 2022 Conference of the North American Chapter of the Association for Computational Linguistics: Human Language Technologies},
  pages={4129--4140},
  year={2022}
}

@article{feng2021nclass-arxiv,
  title={Kgap: Knowledge graph augmented political perspective detection in news media},
  author={Feng, Shangbin and Chen, Zilong and Zhang, Wenqian and Li, Qingyao and Zheng, Qinghua and Chang, Xiaojun and Luo, Minnan},
  journal={arXiv preprint arXiv:2108.03861},
  year={2021}
}

@inproceedings{yasunaga2021kgqa-acl,
  title={QA-GNN: Reasoning with Language Models and Knowledge Graphs for Question Answering},
  author={Yasunaga, Michihiro and Ren, Hongyu and Bosselut, Antoine and Liang, Percy and Leskovec, Jure},
  booktitle={Proceedings of the 2021 Conference of the North American Chapter of the Association for Computational Linguistics: Human Language Technologies},
  pages={535--546},
  year={2021}
}

@inproceedings{zhang2018kgqa-aaai,
  title={Variational reasoning for question answering with knowledge graph},
  author={Zhang, Yuyu and Dai, Hanjun and Kozareva, Zornitsa and Smola, Alexander and Song, Le},
  booktitle={Proceedings of the AAAI conference on artificial intelligence},
  volume={32},
  number={1},
  year={2018}
}

@inproceedings{huang2019kgqa-www,
  title={Knowledge graph embedding based question answering},
  author={Huang, Xiao and Zhang, Jingyuan and Li, Dingcheng and Li, Ping},
  booktitle={Proceedings of the twelfth ACM international conference on web search and data mining},
  pages={105--113},
  year={2019}
}

@inproceedings{safavi2020codex,
  title={CoDEx: A Comprehensive Knowledge Graph Completion Benchmark},
  author={Safavi, Tara and Koutra, Danai},
  booktitle={Proceedings of the 2020 Conference on Empirical Methods in Natural Language Processing (EMNLP)},
  year={2020}
}

@inproceedings{zhang2022redgnn,
  title={Knowledge graph reasoning with relational digraph},
  author={Zhang, Yongqi and Yao, Quanming},
  booktitle={Proceedings of the ACM web conference 2022},
  pages={912--924},
  year={2022}
}

@inproceedings{qi2024bi,
  title={Bi-directional Learning of Logical Rules with Type Constraints for Knowledge Graph Completion},
  author={Qi, Kunxun and Du, Jianfeng and Wan, Hai},
  booktitle={Proceedings of the 33rd ACM International Conference on Information and Knowledge Management},
  pages={1899--1908},
  year={2024}
}

@article{chen2025pathrag,
  title={Pathrag: Pruning graph-based retrieval augmented generation with relational paths},
  author={Chen, Boyu and Guo, Zirui and Yang, Zidan and Chen, Yuluo and Chen, Junze and Liu, Zhenghao and Shi, Chuan and Yang, Cheng},
  journal={arXiv preprint arXiv:2502.14902},
  year={2025}
}

@article{guo2024lightrag,
  title={Lightrag: Simple and fast retrieval-augmented generation},
  author={Guo, Zirui and Xia, Lianghao and Yu, Yanhua and Ao, Tu and Huang, Chao},
  journal={arXiv preprint arXiv:2410.05779},
  year={2024}
}

@inproceedings{lin2020drug-ijcai,
  title={KGNN: Knowledge graph neural network for drug-drug interaction prediction.},
  author={Lin, Xuan and Quan, Zhe and Wang, Zhi-Jie and Ma, Tengfei and Zeng, Xiangxiang},
  booktitle={IJCAI},
  volume={380},
  pages={2739--2745},
  year={2020}
}

@inproceedings{nickel2011rescal,
  title={A three-way model for collective learning on multi-relational data.},
  author={Nickel, Maximilian and Tresp, Volker and Kriegel, Hans-Peter and others},
  booktitle={Icml},
  volume={11},
  number={10.5555},
  pages={3104482--3104584},
  year={2011}
}

@inproceedings{lv2022crat1,
  title={Do pre-trained models benefit knowledge graph completion? a reliable evaluation and a reasonable approach},
  author={Lv, Xin and Lin, Yankai and Cao, Yixin and Hou, Lei and Li, Juanzi and Liu, Zhiyuan and Li, Peng and Zhou, Jie},
  booktitle={Findings of the association for computational linguistics: ACL 2022},
  pages={3570--3581},
  year={2022}
}

@inproceedings{zhou2020rec-sigir,
  title={Interactive recommender system via knowledge graph-enhanced reinforcement learning},
  author={Zhou, Sijin and Dai, Xinyi and Chen, Haokun and Zhang, Weinan and Ren, Kan and Tang, Ruiming and He, Xiuqiang and Yu, Yong},
  booktitle={Proceedings of the 43rd international ACM SIGIR conference on research and development in information retrieval},
  pages={179--188},
  year={2020}
}

@inproceedings{wang2019rec-www,
  title={Knowledge graph convolutional networks for recommender systems},
  author={Wang, Hongwei and Zhao, Miao and Xie, Xing and Li, Wenjie and Guo, Minyi},
  booktitle={The world wide web conference},
  pages={3307--3313},
  year={2019}
}

@inproceedings{wang2018rec-www,
  title={DKN: Deep knowledge-aware network for news recommendation},
  author={Wang, Hongwei and Zhang, Fuzheng and Xie, Xing and Guo, Minyi},
  booktitle={Proceedings of the Web Conference (WWW)},
  pages={1835--1844},
  year={2018}
}

@inproceedings{ko2021mascot,
  title={MASCOT: A Quantization Framework for Efficient Matrix Factorization in Recommender Systems},
  author={Ko, Yunyong and Yu, Jae-Seo and Bae, Hong-Kyun and Park, Yongjun and Lee, Dongwon and Kim, Sang-Wook},
  booktitle={Proceedings of the IEEE International Conference on Data Mining (ICDM)},
  pages={290--299},
  year={2021},
  organization={IEEE}
}

@inproceedings{jang2023sage,
  title={SAGE: A Storage-Based Approach for Scalable and Efficient Sparse Generalized Matrix-Matrix Multiplication},
  author={Jang, Myung-Hwan and Ko, Yunyong and Gwon, Hyuck-Moo and Jo, Ikhyeon and Park, Yongjun and Kim, Sang-Wook},
  booktitle={Proceedings of the ACM International Conference on Information and Knowledge Management (CIKM)},
  pages={923--933},
  year={2023}
}

@inproceedings{vashishth2019composition,
  title     = {Composition-based Multi-Relational Graph Convolutional Networks},
  author    = {Vashishth, Shikhar and Sanyal, Soumya and Nitin, Vikram and Talukdar, Partha},
  booktitle = {Proceedings of International Conference on Learning Representations (ICLR)},
  year      = {2020},
  url       = {https://openreview.net/forum?id=BylA_C4tPr}
}

@inproceedings{trouillon2016complex,
  title     = {Complex embeddings for simple link prediction},
  author    = {Trouillon, Th{\'e}o and Welbl, Johannes and Riedel, Sebastian and Gaussier, {\'E}ric and Bouchard, Guillaume},
  booktitle = {Proceedings of International Conference on Machine Learning},
  pages     = {2071--2080},
  year      = {2016},
  organization = {PMLR}
}

@inproceedings{yang2015embeddingentitiesrelationslearning,
  title     = {Embedding Entities and Relations for Learning and Inference in Knowledge Bases},
  author    = {Yang, Bishan and Yih, Wen-tau and He, Xiaodong and Gao, Jianfeng and Deng, Li},
  booktitle = {Proceedings of International Conference on Learning Representations (ICLR), Poster Track},
  year      = {2015},
  url       = {https://arxiv.org/abs/1412.6575}
}

@inproceedings{toutanova2015observed,
  title     = {Observed versus latent features for knowledge base and text inference},
  author    = {Toutanova, Kristina and Chen, Danqi},
  booktitle = {Proceedings of the 3rd Workshop on Continuous Vector Space Models and Their Compositionality},
  pages     = {57--66},
  year      = {2015}
}

@inproceedings{bollacker2008freebase,
  title     = {Freebase: a collaboratively created graph database for structuring human knowledge},
  author    = {Bollacker, Kurt and Evans, Colin and Paritosh, Praveen and Sturge, Tim and Taylor, Jamie},
  booktitle = {Proceedings of the 2008 ACM SIGMOD International Conference on Management of Data},
  pages     = {1247--1250},
  year      = {2008}
}

@inproceedings{li2022house,
  title     = {House: Knowledge graph embedding with householder parameterization},
  author    = {Li, Rui and Zhao, Jianan and Li, Chaozhuo and He, Di and Wang, Yiqi and Liu, Yuming and Sun, Hao and Wang, Senzhang and Deng, Weiwei and Shen, Yanming and others},
  booktitle = {International Conference on Machine Learning},
  pages     = {13209--13224},
  year      = {2022},
  organization = {PMLR}
}

@inproceedings{liu2021indigo,
  title     = {Indigo: GNN-based inductive knowledge graph completion using pair-wise encoding},
  author    = {Liu, Shuwen and Grau, Bernardo and Horrocks, Ian and Kostylev, Egor},
  booktitle = {Advances in Neural Information Processing Systems},
  volume    = {34},
  pages     = {2034--2045},
  year      = {2021}
}

@article{guo2020survey,
  title     = {A survey on knowledge graph-based recommender systems},
  author    = {Guo, Qingyu and Zhuang, Fuzhen and Qin, Chuan and Zhu, Hengshu and Xie, Xing and Xiong, Hui and He, Qing},
  journal   = {IEEE Transactions on Knowledge and Data Engineering},
  volume    = {34},
  number    = {8},
  pages     = {3549--3568},
  year      = {2020},
  publisher = {IEEE}
}

@inproceedings{zhu2021neural,
  title     = {Neural Bellman-Ford networks: A general graph neural network framework for link prediction},
  author    = {Zhu, Zhaocheng and Zhang, Zuobai and Xhonneux, Louis-Pascal and Tang, Jian},
  booktitle = {Advances in Neural Information Processing Systems},
  volume    = {34},
  pages     = {29476--29490},
  year      = {2021}
}

@inproceedings{carlson2010toward,
  title     = {Toward an architecture for never-ending language learning},
  author    = {Carlson, Andrew and Betteridge, Justin and Kisiel, Bryan and Settles, Burr and Hruschka, Estevam and Mitchell, Tom},
  booktitle = {Proceedings of the AAAI Conference on Artificial Intelligence},
  volume    = {24},
  number    = {1},
  pages     = {1306--1313},
  year      = {2010}
}

@inproceedings{yang2017differentiable,
  title     = {Differentiable learning of logical rules for knowledge base reasoning},
  author    = {Yang, Fan and Yang, Zhilin and Cohen, William W},
  booktitle = {Advances in Neural Information Processing Systems},
  volume    = {30},
  year      = {2017}
}

@inproceedings{yang2022rethinking,
  title     = {Rethinking knowledge graph evaluation under the open-world assumption},
  author    = {Yang, Haotong and Lin, Zhouchen and Zhang, Muhan},
  booktitle = {Advances in Neural Information Processing Systems},
  volume    = {35},
  pages     = {8374--8385},
  year      = {2022}
}

@inproceedings{wang2019tackling,
  title     = {Tackling Long-Tailed Relations and Uncommon Entities in Knowledge Graph Completion},
  author    = {Wang, Zihao and Lai, Kwunping and Li, Piji and Bing, Lidong and Lam, Wai},
  booktitle = {Proceedings of the 2019 Conference on Empirical Methods in Natural Language Processing and the 9th International Joint Conference on Natural Language Processing (EMNLP-IJCNLP)},
  month     = {nov},
  year      = {2019},
  address   = {Hong Kong, China},
  publisher = {Association for Computational Linguistics},
  url       = {https://aclanthology.org/D19-1024/},
  doi       = {10.18653/v1/D19-1024},
  pages     = {250--260}
}

@inproceedings{qu2019probabilistic,
  title     = {Probabilistic logic neural networks for reasoning},
  author    = {Qu, Meng and Tang, Jian},
  booktitle = {Advances in Neural Information Processing Systems},
  volume    = {32},
  year      = {2019}
}

@inproceedings{mohamed2020popularity,
  title     = {Popularity agnostic evaluation of knowledge graph embeddings},
  author    = {Mohamed, Aisha and Parambath, Shameem and Kaoudi, Zoi and Aboulnaga, Ashraf},
  booktitle = {Conference on Uncertainty in Artificial Intelligence},
  pages     = {1059--1068},
  year      = {2020},
  organization = {PMLR}
}

@inproceedings{yih2014semantic,
  title     = {Semantic parsing for single-relation question answering},
  author    = {Yih, Wen-tau and He, Xiaodong and Meek, Christopher},
  booktitle = {Proceedings of the 52nd Annual Meeting of the Association for Computational Linguistics (Volume 2: Short Papers)},
  pages     = {643--648},
  year      = {2014}
}

@inproceedings{hao2017end,
  title     = {An end-to-end model for question answering over knowledge base with cross-attention combining global knowledge},
  author    = {Hao, Yanchao and Zhang, Yuanzhe and Liu, Kang and He, Shizhu and Liu, Zhanyi and Wu, Hua and Zhao, Jun},
  booktitle = {Proceedings of the 55th Annual Meeting of the Association for Computational Linguistics (Volume 1: Long Papers)},
  pages     = {221--231},
  year      = {2017}
}

@article{edge2024local,
  title     = {From local to global: A graph RAG approach to query-focused summarization},
  author    = {Edge, Darren and Trinh, Ha and Cheng, Newman and Bradley, Joshua and Chao, Alex and Mody, Apurva and Truitt, Steven and Metropolitansky, Dasha and Ness, Robert Osazuwa and Larson, Jonathan},
  journal   = {arXiv preprint arXiv:2404.16130},
  year      = {2024}
}

@article{pan2024unifying,
  title     = {Unifying large language models and knowledge graphs: A roadmap},
  author    = {Pan, Shirui and Luo, Linhao and Wang, Yufei and Chen, Chen and Wang, Jiapu and Wu, Xindong},
  journal   = {IEEE Transactions on Knowledge and Data Engineering},
  volume    = {36},
  number    = {7},
  pages     = {3580--3599},
  year      = {2024},
  publisher = {IEEE}
}

@article{bonner2022implications,
  title={Implications of topological imbalance for representation learning on biomedical knowledge graphs},
  author={Bonner, Stephen and Kirik, Ufuk and Engkvist, Ola and Tang, Jian and Barrett, Ian P},
  journal={Briefings in bioinformatics},
  volume={23},
  number={5},
  pages={bbac279},
  year={2022},
  publisher={Oxford University Press}
}

@inproceedings{shomer2023toward,
  title={Toward degree bias in embedding-based knowledge graph completion},
  author={Shomer, Harry and Jin, Wei and Wang, Wentao and Tang, Jiliang},
  booktitle={Proceedings of the ACM web conference 2023},
  pages={705--715},
  year={2023}
}

@article{zhu2023mednewdisease,
  title={RDKG-115: Assisting drug repurposing and discovery for rare diseases by trimodal knowledge graph embedding},
  author={Zhu, Chaoyu and Xia, Xiaoqiong and Li, Nan and Zhong, Fan and Yang, Zhihao and Liu, Lei},
  journal={Computers in Biology and Medicine},
  volume={164},
  pages={107262},
  year={2023},
  publisher={Elsevier}
}

@article{ye2021medcoldstart,
  title={A unified drug--target interaction prediction framework based on knowledge graph and recommendation system},
  author={Ye, Qing and Hsieh, Chang-Yu and Yang, Ziyi and Kang, Yu and Chen, Jiming and Cao, Dongsheng and He, Shibo and Hou, Tingjun},
  journal={Nature communications},
  volume={12},
  number={1},
  pages={6775},
  year={2021},
  publisher={Nature Publishing Group UK London}
}

@article{bang2023biomedical,
  title={Biomedical knowledge graph learning for drug repurposing by extending guilt-by-association to multiple layers},
  author={Bang, Dongmin and Lim, Sangsoo and Lee, Sangseon and Kim, Sun},
  journal={Nature Communications},
  volume={14},
  number={1},
  pages={3570},
  year={2023},
  publisher={Nature Publishing Group UK London}
}

@article{vella2022medtrialcost,
  title={Few-shot learning for low-data drug discovery},
  author={Vella, Daniel and Ebejer, Jean-Paul},
  journal={Journal of Chemical Information and Modeling},
  volume={63},
  number={1},
  pages={27--42},
  year={2022},
  publisher={ACS Publications}
}

@inproceedings{jiang2025medRAG,
  title={HyKGE: A hypothesis knowledge graph enhanced RAG framework for accurate and reliable medical LLMs responses},
  author={Jiang, Xinke and Zhang, Ruizhe and Xu, Yongxin and Qiu, Rihong and Fang, Yue and Wang, Zhiyuan and Tang, Jinyi and Ding, Hongxin and Chu, Xu and Zhao, Junfeng and others},
  booktitle={Proceedings of the 63rd Annual Meeting of the Association for Computational Linguistics (Volume 1: Long Papers)},
  pages={11836--11856},
  year={2025}
}

@article{matsumoto2024bioRAG,
  title={KRAGEN: a knowledge graph-enhanced RAG framework for biomedical problem solving using large language models},
  author={Matsumoto, Nicholas and Moran, Jay and Choi, Hyunjun and Hernandez, Miguel E and Venkatesan, Mythreye and Wang, Paul and Moore, Jason H},
  journal={Bioinformatics},
  volume={40},
  number={6},
  pages={btae353},
  year={2024},
  publisher={Oxford University Press}
}

@inproceedings{luo2025hypergraphrag,
  title={Hypergraphrag: Retrieval-augmented generation via hypergraph-structured knowledge representation},
  author={Luo, Haoran and Chen, Guanting and Zheng, Yandan and Wu, Xiaobao and Guo, Yikai and Lin, Qika and Feng, Yu and Kuang, Zemin and Song, Meina and Zhu, Yifan and others},
  journal={arXiv preprint arXiv:2503.21322},
  year={2025}
}

@inproceedings{schlichtkrull2018modeling,
  title     = {Modeling relational data with graph convolutional networks},
  author    = {Schlichtkrull, Michael and Kipf, Thomas N and Bloem, Peter and Van Den Berg, Rianne and Titov, Ivan and Welling, Max},
  booktitle = {The Semantic Web: 15th International Conference, ESWC 2018, Heraklion, Crete, Greece, June 3--7, 2018, Proceedings 15},
  pages     = {593--607},
  year      = {2018},
  organization = {Springer}
}

@inproceedings{qu2020rnnlogic,
  title     = {RNNLogic: Learning Logic Rules for Reasoning on Knowledge Graphs},
  author    = {Qu, Meng and Chen, Junkun and Xhonneux, Louis-Pascal and Bengio, Yoshua and Tang, Jian},
  booktitle = {International Conference on Learning Representations (ICLR)},
  year      = {2021},
  url       = {https://openreview.net/forum?id=tGZu6DlbreV}
}

@inproceedings{sun2019rotate,
  title     = {RotatE: Knowledge Graph Embedding by Relational Rotation in Complex Space},
  author    = {Sun, Zhiqing and Deng, Zhi-Hong and Nie, Jian-Yun and Tang, Jian},
  booktitle = {International Conference on Learning Representations (ICLR)},
  year      = {2019},
  url       = {https://openreview.net/forum?id=HkgEQnRqYQ}
}

@article{zhang2021drug,
  title     = {Drug repurposing for COVID-19 via knowledge graph completion},
  author    = {Zhang, Rui and Hristovski, Dimitar and Schutte, Dalton and Kastrin, Andrej and Fiszman, Marcelo and Kilicoglu, Halil},
  journal   = {Journal of Biomedical Informatics},
  volume    = {115},
  pages     = {103696},
  year      = {2021},
  publisher = {Elsevier}
}

@article{zeng2022toward,
  title     = {Toward better drug discovery with knowledge graph},
  author    = {Zeng, Xiangxiang and Tu, Xinqi and Liu, Yuansheng and Fu, Xiangzheng and Su, Yansen},
  journal   = {Current Opinion in Structural Biology},
  volume    = {72},
  pages     = {114--126},
  year      = {2022},
  publisher = {Elsevier}
}

@inproceedings{xiong2018oneshot,
  title     = {One-Shot Relational Learning for Knowledge Graphs},
  author    = {Xiong, Wenhan and Yu, Mo and Chang, Shiyu and Guo, Xiaoxiao and Wang, William Yang},
  booktitle = {Proceedings of the 2018 Conference on Empirical Methods in Natural Language Processing},
  month     = {oct--nov},
  year      = {2018},
  address   = {Brussels, Belgium},
  publisher = {Association for Computational Linguistics},
  url       = {https://aclanthology.org/D18-1223/},
  doi       = {10.18653/v1/D18-1223},
  pages     = {1980--1990}
}

@inproceedings{sun2020reevaluation,
  title     = {A Re-evaluation of Knowledge Graph Completion Methods},
  author    = {Sun, Zhiqing and Vashishth, Shikhar and Sanyal, Soumya and Talukdar, Partha and Yang, Yiming},
  booktitle = {Proceedings of the 58th Annual Meeting of the Association for Computational Linguistics},
  month     = {jul},
  year      = {2020},
  address   = {Online},
  publisher = {Association for Computational Linguistics},
  url       = {https://aclanthology.org/2020.acl-main.489/},
  doi       = {10.18653/v1/2020.acl-main.489},
  pages     = {5516--5522}
}

@inproceedings{bordes2013translating,
  title     = {Translating embeddings for modeling multi-relational data},
  author    = {Bordes, Antoine and Usunier, Nicolas and Garcia-Duran, Alberto and Weston, Jason and Yakhnenko, Oksana},
  booktitle = {Advances in Neural Information Processing Systems},
  volume    = {26},
  year      = {2013}
}

@inproceedings{balazevic2019tucker,
  title     = {TuckER: Tensor Factorization for Knowledge Graph Completion},
  author    = {Balazevic, Ivana and Allen, Carl and Hospedales, Timothy},
  booktitle = {Proceedings of the 2019 Conference on Empirical Methods in Natural Language Processing and the 9th International Joint Conference on Natural Language Processing (EMNLP-IJCNLP)},
  month     = {nov},
  year      = {2019},
  address   = {Hong Kong, China},
  publisher = {Association for Computational Linguistics},
  url       = {https://aclanthology.org/D19-1522/},
  doi       = {10.18653/v1/D19-1522},
  pages     = {5185--5194}
}

@article{hoyt2022unified,
  title     = {A unified framework for rank-based evaluation metrics for link prediction in knowledge graphs},
  author    = {Hoyt, Charles Tapley and Berrendorf, Max and Galkin, Mikhail and Tresp, Volker and Gyori, Benjamin M},
  journal   = {arXiv preprint arXiv:2203.07544},
  year      = {2022}
}

@inproceedings{dettmers2018convolutional,
  title     = {Convolutional 2D knowledge graph embeddings},
  author    = {Dettmers, Tim and Minervini, Pasquale and Stenetorp, Pontus and Riedel, Sebastian},
  booktitle = {Proceedings of the AAAI Conference on Artificial Intelligence},
  volume    = {32},
  number    = {1},
  year      = {2018}
}

@article{miller1995wordnet,
  title     = {WordNet: a lexical database for English},
  author    = {Miller, George A},
  journal   = {Communications of the ACM},
  volume    = {38},
  number    = {11},
  pages     = {39--41},
  year      = {1995},
  publisher = {ACM}
}

@inproceedings{suchanek2007yago,
  title     = {YAGO: a core of semantic knowledge},
  author    = {Suchanek, Fabian M and Kasneci, Gjergji and Weikum, Gerhard},
  booktitle = {Proceedings of the 16th International Conference on World Wide Web},
  pages     = {697--706},
  year      = {2007}
}

@inproceedings{mahdisoltani2013yago3,
  title     = {YAGO3: A knowledge base from multilingual Wikipedias},
  author    = {Mahdisoltani, Farzaneh and Biega, Joanna and Suchanek, Fabian M},
  booktitle = {CIDR},
  year      = {2013}
}

@article{sun2024mlsaa,
  title={Incorporating multi-level sampling with adaptive aggregation for inductive knowledge graph completion},
  author={Sun, Kai and Jiang, Huajie and Hu, Yongli and Yin, Baocai},
  journal={ACM transactions on knowledge discovery from data},
  volume={18},
  number={5},
  pages={1--16},
  year={2024},
  publisher={ACM New York, NY}
}

@inproceedings{moon2025sharp,
  title={How Sharp and Bias-Robust is a Model? Dual Evaluation Perspectives on Knowledge Graph Completion},
  author={Moon, Sooho and Ko, Yunyong},
  booktitle={Proceedings of the nineteenth ACM international conference on web search and data mining},
  year={2026}
}

@inproceedings{tiwari2021revisiting,
  title={Revisiting the evaluation protocol of knowledge graph completion methods for link prediction},
  author={Tiwari, Sudhanshu and Bansal, Iti and Rivero, Carlos R},
  booktitle={Proceedings of the Web Conference 2021},
  pages={809--820},
  year={2021}
}

@article{berrendorf2020ambiguity,
  title={On the ambiguity of rank-based evaluation of entity alignment or link prediction methods},
  author={Berrendorf, Max and Faerman, Evgeniy and Vermue, Laurent and Tresp, Volker},
  journal={arXiv preprint arXiv:2002.06914},
  year={2020}
}

@article{bonner2022review,
  title     = {A review of biomedical datasets relating to drug discovery: a knowledge graph perspective},
  author    = {Bonner, Stephen and Barrett, Ian P and Ye, Cheng and Swiers, Rowan and Engkvist, Ola and Bender, Andreas and Hoyt, Charles Tapley and Hamilton, William L},
  journal   = {Briefings in Bioinformatics},
  volume    = {23},
  number    = {6},
  pages     = {bbac404},
  year      = {2022},
  publisher = {Oxford University Press}
}

% \appendix

% \section{Interpretation Under PROBE}
% \begin{lemma}[h]
%     Hits@1 is equal to $PROBE_{\alpha \rightarrow \infty}$
% \end{lemma}
% \begin{align*}
%     PROBE_{\alpha \rightarrow \infty}=\lim_{\alpha \rightarrow \infty} \frac{r^{-\alpha} - \E^{-\alpha}}{1-\E^{-\alpha}}  \\
%     =\begin{cases*}
%     1 & \text{if } r = 1 \\
%     0 & \text{if } r > 1
%     \end{cases*} \\
% \end{align*}

% \begin{lemma}
%     $PROBE_{\alpha =-1}$ is a chance adjusted metric
% \end{lemma}
% Let's assume $r_i\sim Uniform(\left\{1,2,...,\E\right\})$ where $\mathbb{E}[r_i]=\frac{\E+1}{2}$
% \begin{align*}
%     \mathbb{E}[PROBE_{\alpha =-1}]
%     &=\mathbb{E}[\frac{1}{n}\sum_{i=1}^n \frac{r_i-\E}{1-\E}] \\
%     &=\frac{1}{n}\sum_{i=1}^n \mathbb{E}[\frac{r_i-\E}{1-\E}] \\
%     &=\frac{1}{n}\sum_{i=1}^n 
%     \frac{\mathbb{E}[r_i]-\E}{1-\E} \\
%     &=\frac{1}{n}\sum_{i=1}^n 
%     \frac{1-\E}{2(1-\E)} \\
%     &=\frac{1}{2}.
% \end{align*}

% \section{Property Proofs}

% \section{Generalization of PROBE}
% \begin{itemize}
%     \item \textbf{MR}: Equivalent to PROBE with transformation function $f$ and $\alpha=-1$.
%     \item \textbf{MRR}: Equivalent to PROBE with transformation function $f$ and $\alpha=1$.
%     \item \textbf{p-MRR}: Equivalent to PROBE with transformation function $f$ and treat $p$ as $\alpha$.
%     \item \textbf{Hits@K}: Although a direct form can't be achieved from our PROBE, we can formulate Hits@K as a $f^*$ with $\alpha\rightarrow\infty$ that is shifted $K$ amount along the rank axis.
% \end{itemize}

\end{document}